\documentclass[accepted]{uai2024} %
\usepackage[american]{babel}
\usepackage{natbib} %
\usepackage{mathtools} %
\usepackage{booktabs} 
\usepackage{tikz} 

\newcommand{\BudgetNodesRSIDsparse}{256}
\newcommand{\BudgetNodesRSIDdense}{239}

\newcommand{\RuntimeMeanOsetsparse}{567 ms}

\newcommand{\RuntimeMeanAncestorsparse}{40.9 ms}

\newcommand{\RuntimeMeanParentsparse}{30.5 ms}

\newcommand{\RuntimeMeanOsetdense}{1.68 s}

\newcommand{\RuntimeMeanAncestordense}{207 ms}

\newcommand{\RuntimeMeanParentdense}{173 ms}
\newcommand{\BudgetNodesOsetsparse}{1105}
\newcommand{\BudgetNodesOsetdense}{508}

\newcommand{\BudgetNodesAncestorsparse}{8211}
\newcommand{\BudgetNodesAncestordense}{932}

\newcommand{\BudgetNodesParentsparse}{13601}
\newcommand{\BudgetNodesParentdense}{962}

\newcommand{\ThousandMeanOset}{5.0 ms}
\newcommand{\ThousandMeanAncestor}{3.4 ms}
\newcommand{\ThousandMeanParent}{7.3 ms}
\usepackage{placeins}

\usepackage{microtype}

\hypersetup{
    hidelinks,
    colorlinks,
    citecolor=green!50!black, %
    linkcolor=red!50!black, %
    urlcolor=cyan!50!black, %
}

\newcommand{\walknoncausal}{\texttt{\textsc{NonCausal}}}
\newcommand{\walkcausalopen}{\texttt{\textsc{PD$^{T\to}$Open}}}
\newcommand{\walkcausalblocked}{\texttt{\textsc{PD$^{T\to}$Blocked}}}
\newcommand{\walknamopen}{\texttt{\textsc{PD$^{T\arrow{}}$Open}}}
\newcommand{\walknamblocked}{\texttt{\textsc{PD$^{T\arrow{}}$Blocked}}}

\newcommand{\arrow}{\mathrel{\ooalign{\rule[.5ex]{1.8ex}{.45pt}}}}

\newcommand{\doop}[1]{\operatorname{do}\mleft(#1\mright)}

\newcommand{\g}[1][G]{\mathcal{#1}}
\newcommand{\G}{\mathcal{G}}
\newcommand{\Gtrue}{\mathcal{G}_\text{true}}
\newcommand{\Gguess}{\mathcal{G}_\text{guess}}

\newcommand{\none}{\mathsf{none}}
\newcommand{\correct}{\mathsf{correct}}
\newcommand{\false}{\mathsf{incorrect}}

\DeclareMathOperator{\pa}{Pa}
\DeclareMathOperator{\posspa}{PossPa}
\DeclareMathOperator{\ch}{Ch}
\DeclareMathOperator{\possch}{PossCh}
\DeclareMathOperator{\an}{An}
\DeclareMathOperator{\possan}{PossAn}
\DeclareMathOperator{\de}{De}
\DeclareMathOperator{\possde}{PossDe}
\DeclareMathOperator{\forb}{Forb}
\DeclareMathOperator{\possDe}{PossDe}
\DeclareMathOperator{\CN}{Cn}
\DeclareMathOperator{\possCN}{PossCn}
\DeclareMathOperator{\adjacents}{AdjacentNodes}
\newcommand{\cn}[2][T,Y]{\CN(#1,#2)}
\newcommand{\cnb}[2][T,Y]{\CN(\mathbf{#1},#2)}
\newcommand{\posscnb}[2][T,Y]{\possCN(\mathbf{#1},#2)}
\newcommand{\opt}[2][T,Y]{\mathbf{O}(#1,#2)}
\newcommand{\optb}[2][T,Y]{\mathbf{O}(\mathbf{#1},#2)}

\usepackage{tikz}
\usetikzlibrary{
    arrows, arrows.meta,
    automata,
    calc,
    decorations, decorations.pathreplacing,
    fit,
    patterns,
    positioning,
    quotes,
    shapes,
}
\tikzset{
    bidirected/.style={Latex-Latex,dashed},
    directed/.style={-Latex,semithick},
    el/.style = {inner sep=2pt, align=left, sloped}
    point/.style = {circle, draw, inner sep=0.04cm,fill,node contents={}},
    state/.style ={minimum width=0.5cm},
}

\usepackage[labelformat=simple]{subcaption}

\usepackage{amsmath,amssymb,amsthm}
\usepackage{bbold}
\usepackage{mleftright}

\usepackage[capitalise,noabbrev,nameinlink]{cleveref}
\crefname{Definition}{Definition}{Definitions}
\crefname{Example}{Example}{Examples}
  
\newtheorem{Satz}{Satz}
\theoremstyle{plain}
\newtheorem{Lemma}[Satz]{Lemma}
\newtheorem{Corollary}[Satz]{Corollary}

\newtheorem{Proposition}[Satz]{Proposition}

\theoremstyle{definition}
\makeatletter
\def\th@definition{%
  \thm@notefont{\normalfont\itshape}%
  \normalfont %
}
\makeatother
\newtheorem{Example}[Satz]{Example}
\newtheorem{Definition}[Satz]{Definition}

\usepackage{algorithm}
\usepackage[noend]{algpseudocode}

\usepackage{etoolbox}

\makeatletter
\newcommand*{\algrule}[1][\algorithmicindent]{%
  \makebox[#1][l]{%
    \hspace*{.2em}%
    {\color{gray!50!white}\vrule height .75\baselineskip depth .25\baselineskip width 1pt}
  }
}

\newcount\ALG@printindent@tempcnta
\def\ALG@printindent{%
    \ifnum \theALG@nested>0%
    \ifx\ALG@text\ALG@x@notext%
    \else
    \unskip
    \ALG@printindent@tempcnta=1
    \loop
    \algrule[\csname ALG@ind@\the\ALG@printindent@tempcnta\endcsname]%
    \advance \ALG@printindent@tempcnta 1
    \ifnum \ALG@printindent@tempcnta<\numexpr\theALG@nested+1\relax
    \repeat
    \fi
    \fi
}
\patchcmd{\ALG@doentity}{\noindent\hskip\ALG@tlm}{\ALG@printindent}{}{\errmessage{failed to patch}}
\patchcmd{\ALG@doentity}{\item[]\nointerlineskip}{}{}{} %
\makeatother
\algrenewcommand\Call[2]{%
  \textproc{#1}%
  \ifthenelse{\equal{\detokenize{#2}}{}}{}{(#2)}%
}

\newcommand{\Title}{Adjustment Identification Distance: A \texttt{gadjid} for Causal Structure Learning}
\title{\Title}

\author[1]{\href{mailto:<leonard.henckel@ucd.ie>?Subject=Your UAI 2024 paper - gadjid}{Leonard Henckel}{}}
\author[2]{Theo Würtzen}
\author[2,3]{Sebastian Weichwald}
\affil[1]{%
    School of Mathematics and Statistics,
    University College Dublin,
    Ireland
}
\affil[2]{%
     Pioneer Centre for AI,
     University of Copenhagen,
     Denmark
}
\affil[3]{%
    Department of Mathematical Sciences,
    University of Copenhagen,
    Denmark
}

\begin{document}

\maketitle

\begin{abstract}
Evaluating graphs learned by causal discovery algorithms is difficult: The number of edges that differ between two graphs does not reflect how the graphs differ with respect to the identifying formulas they suggest for causal effects. We introduce a framework for developing causal distances between graphs which includes the structural intervention distance for directed acyclic graphs as a special case. We use this framework to develop improved adjustment-based distances as well as extensions to completed partially directed acyclic graphs and causal orders. We develop new reachability algorithms to compute the distances efficiently
and to prove their low polynomial time complexity.
In our package \texttt{gadjid} (open source at \href{https://github.com/CausalDisco/gadjid}{github.com/CausalDisco/gadjid}), we 
provide implementations of our distances; they are orders of magnitude faster with proven lower time complexity than the structural intervention distance and thereby provide a success metric for causal discovery that scales to graph sizes that were previously prohibitive.
\end{abstract}

\section{Introduction}\label{sec:intro}

Inferring the causal effect of a treatment on an outcome from observational data requires qualitative knowledge of the underlying causal structure, for instance in form of a causal graph \citep{pearl2009causality}.
We can, for example, use a causal graph to decide whether a set of covariates forms a valid adjustment set
and enables correct estimation of a causal effect via adjustment
\citep{pearl1995causal,perkovic2018complete}.

Under certain assumptions, we can learn the causal graph that underlies the covariates; a task known as causal discovery \citep[e.g.\@][]{spirtes2000causation,Chickering02,heinze2018causal}. Causal discovery is challenging. First, the causal graph is only identifiable from observational data under restrictive assumptions, such as additive errors satisfying distributional or scale restrictions \citep{shimizu2006linear,park2020identifiability}.
Second,
algorithms
based on existing identifiability results
often require
further assumptions or approximations
to be computationally feasible,
for example,
testing only few of
the combinatorially exploding number of required conditional independence tests
\citep{spirtes2000causation}.
Third, causal discovery from finite data is statistically challenging
and there are pitfalls to evaluating causal discovery algorithms on simulated data \citep{gentzel2019case,weichwald2020causal,kaiser2022unsuitability,reisach2021beware,reisach2023scale}.

The literature has focused on the first two problems.
Yet, for causal discovery to become practically useful,
it is necessary to tackle the third problem and improve the evaluation criteria, benchmarks, and success metrics that guide algorithm development \citep{mooij2016distinguishing,cheng2022evaluation,rios2023benchpress}.
A prerequisite for research into more accurate causal discovery algorithms
is that we quantify that accuracy,
for which we need a distance between a learned graph $\Gguess$ and the true graph $\Gtrue$.
A common and very widely used choice in the literature is the structural Hamming distance (SHD) or variants thereof
which count the number of edges that differ between graphs \citep{tsamardinos2006max,de2009comparison,constantinou2019evaluating}.
The SHD, however, does not reflect how similar graphs are when used to infer interventional distributions:
A graph $\Gguess$ may have a large Hamming distance from $\Gtrue$ but still be a good estimate of $\Gtrue$ for performing causal inference (cf.\@ \cref{corollary: hamming} or \citep{peters2015structural}).
The number of edges that differ between graphs is not a performance metric for causal discovery when the graph is to be used for effect identification.

The literature on comparing causal graphs can broadly be divided into two approaches.
The first approach
considers data-driven graph distances
\citep{viinikka2018intersection,eigenmann2020evaluation,peyrard2021ladder,dhanakshirur2023continuous}.
These are challenging to use as performance metrics for algorithm development, as evaluating these distances generally requires large samples and is only computationally feasible for small graphs.
The second approach
considers only the graph structure and its implications for causal inference to define a graph distance;
an example is the
Structural Intervention Distance (SID) \citep{peters2015structural}.
We focus on distances that consider only the graph structure.
This approach has received less attention than the first approach but offers advantages:
First, it enables comparisons to graphs encoding expert knowledge without having to specify all conditional distributions.
Second, it is independent of the sample size, hyperparameter tuning, and choice of density estimator.

The SID counts interventional densities $f(Y\mid \doop{T=t})$ in $\Gtrue$ that are incorrectly inferred if we instead use $\Gguess$ as the causal graph to compute interventional densities via parent adjustment.
For directed acyclic graphs (DAGs), this amounts to counting the parent sets in $\Gguess$ that are not valid adjustment sets in $\Gtrue$.
Using this characterization, \citet{peters2015structural} provide an algorithm 
with $O(p^4\log(p))$ time complexity in the number of nodes $p$.\footnote{%
For certain adjacency matrices, using the Strassen algorithm for matrix multiplication in the algorithm by \citet{peters2015structural} may reduce the complexity to $O(p^{\log_2(7)+1}\log(p))\approx O(p^{3.8}\log(p))$;
using our novel reachability algorithms we reduce the complexity to $O(p^3)$ for dense and $O(p^2)$ for sparse graphs.
}
They propose a generalization of the SID to completed partially directed acyclic graphs (CPDAGs), which represent equivalence classes of DAGs, by iterating over the DAGs in the equivalence class to compute a multi-set of distances.
As one major use of causal graphs is inferring interventional densities,
the SID is a practically relevant distance.
However, parent adjustment is only one of many approaches to compute causal effects and is in fact statistically inefficient \citep{rotnitzky2020efficient,henckel2019graphical}.
Further, iterating over the DAGs in a Markov equivalence class to calculate the multi-set SID between CPDAGs has exponential time complexity and the resulting non-scalar distance is difficult to interpret. In fact, the CRAN SID package v1.1 requires that the true graph be a directed acyclic graph, that is, it does not implement the distance between two completed partially directed acyclic graphs outlined by \citet{peters2015structural};
calculating the SID between a DAG and CPDAG exactly likewise has exponential time complexity and returns a difficult to interpret multi-set.

\paragraph{Contribution.}
We develop distances between causal graphs that reflect
their dissimilarity when used to infer causal effects.
Specifically,
we propose a framework to construct an identifiability distance from a graphical identification strategy, that is, an algorithmic approach to causal effect identification. 
We posit that such a distance is interesting if the
underlying graphical strategy for causal effect identification
represents a potential practitioner using the graph to infer causal effects. Different assumptions on how an idealized practitioner would use a causal graph lead to different causal graph distances.
We show that our framework 
includes the SID for DAGs as a special case
and use the framework to propose new distances with attractive properties.
We discuss for each distance a) whether the underlying identification strategy is good practice and used in practice, and b) when it is zero.
Within our framework and in contrast to the SID,
our distances canonically generalize to distances between any combination of DAGs and CPDAGs.
We also generalize one of the distances to learned causal node orders.
We develop polynomial time algorithms to compute the distances, which,
to our knowledge, makes them the first causal distances between CPDAGs with a polynomial runtime guarantee.
For the distances using local adjustment strategies,
which in the special case of DAGs includes the SID,
we show that the complexity is $O(p^2)$ for sparse and $O(p^3)$ for dense graphs, irrespective of whether the graphs are DAGs or CPDAGs.
We also show that the complexity of our most complex distance is at most $O(p^4)$.
We provide empirical evidence for the asymptotic time complexities and fast runtimes of our algorithms. 
Finally, we discuss how future advances on sound and complete criteria for causal effect identification could be integrated into our framework to develop distances for more general graph types that allow for unobserved variables.

\paragraph{A \texttt{gadjid} for causal structure learning.}
For our distances, we provide efficient 
Rust-implementations with a Python interface.
Our package
\texttt{gadjid} (open source at \href{https://github.com/CausalDisco/gadjid}{github.com/CausalDisco/gadjid})
enables
researchers to evaluate and benchmark causal discovery algorithms
with causally meaningful and computationally tractable performance metrics
to guide and support the development of
structure learning algorithms.

\section{Preliminaries}

\begin{figure}[t!]
\centering
    \begin{tikzpicture}[scale=1]
      \node[state] (w1) at (-1.5,0) {$V_2$};
      \node[state] (w2) at (0,0) {$V_3$};
      \node[state] (w3) at (-3,0) {$V_1$};
      \node[state] (w4) at (-.75,1) {$V_4$};
      \path (w1) edge[-] (w2)
            (w3) edge[directed] (w1)
            (w4) edge[directed] (w1)
            (w4) edge[directed] (w2)
            (w3) edge[directed,bend right=20] (w2);
      \node[state] (v1) at (3,1) {$V_2$};
      \node[state] (v2) at (4.5,1) {$V_3$};
      \node[state] (v3) at (1.5,1) {$V_1$};
      \node[state] (v4) at (3.75,2) {$V_4$};
      \node[state] (v5) at (3,-1) {$V_2$};
      \node[state] (v6) at (4.5,-1) {$V_3$};
      \node[state] (v7) at (1.5,-1) {$V_1$};
      \node[state] (v8) at (3.75,-0) {$V_4$};
      \path (v1) edge[directed] (v2)
            (v3) edge[directed] (v1)
            (v4) edge[directed] (v1)
            (v4) edge[directed] (v2)
            (v3) edge[directed,bend right=20] (v2);
        \path (v6) edge[directed] (v5)
            (v7) edge[directed] (v5)
            (v8) edge[directed] (v5)
            (v8) edge[directed] (v6)
            (v7) edge[directed,bend right=20] (v6);
        \draw[dashed] (.75,-1.4) -- (.75,2.17);
    \end{tikzpicture}
    \vspace{-.2cm}
    \caption{A CPDAG (left) and the two DAGs in the corresponding Markov equivalence class (right).}
    \label{figure: amenability example}
\end{figure}

We use graphs where nodes represent random variables
and edges causal relationships. 
Here, we provide an overview of the key terminology 
and refer to \cref{supp:preliminaries} for details.

We consider two types of graphs: directed acyclic graphs (DAGs) and completed partially directed acyclic graphs (CPDAGs), see \cref{figure: amenability example}. DAGs are graphs
with directed edges ($\rightarrow$) and without directed cycles. 
DAGs can describe causal
relationships
without feedback loops
\citep{pearl2009causality}.
They also encode conditional independences that can be read off the graph using the d-separation criterion \citep{pearl2009causality}. 
DAGs can be learned from data only under strong assumptions.
However, the class of DAGs encoding the same conditional independences, known as its Markov equivalence class, can be learned under weaker assumptions.
A CPDAG can uniquely represent this equivalence class if there are no hidden variables
\citep{meek1995causal,Chickering02}. CPDAGs contain directed ($\rightarrow$) and undirected ($\arrow$) edges and satisfy further structural properties \citep{meek1995causal}.

\paragraph{Causal DAGs and CPDAGs.}
We consider external interventions $\doop{\mathbf{T}=\mathbf{t}}$ (short $\doop{\mathbf{t}}$) for $\mathbf{T} \subseteq \mathbf{V}$ that set $\mathbf{T}$ to some value $\mathbf{t}$ for the entire population \citep{pearl1995causal}. 
A probability density function $f$ over random variables
$\mathbf{V} = (V_1, \dots, V_p)$
is compatible with a causal DAG $\g = (\mathbf{V}, \mathbf{E})$ if all densities $f(\mathbf{v}\mid\doop{\mathbf{t}})$
obey
\begin{equation*} %
\displaystyle
    f(\mathbf{v}\mid\doop{\mathbf{t}}) = 
    \begin{cases}
        \prod_{V \in \mathbf{V} \setminus \mathbf{T}} f(v\mid \pa(V,\g))& \text{if } \mathbf{T}=\mathbf{t}, \\
        0 & \text{otherwise.}
    \end{cases}
\end{equation*}
This equation is known as truncated factorization formula \citep{pearl2009causality}, manipulated density
formula \citep{spirtes2000causation}, or g-formula \citep{robins1986new}. A density $f$ is
compatible with a CPDAG $\g$ if it is compatible with a causal DAG in the Markov equivalence class represented by $\g$.

\paragraph{Identifying formula.}
Causal graphs are used to estimate the causal effect of a treatment $\mathbf{T}\subseteq \mathbf{V}$ on an outcome $\mathbf{Y}\subseteq\mathbf{V}$ from observational data, that is, to estimate (functionals of) the interventional distribution
$f(\mathbf{y}\mid \doop{\mathbf{t}})$.
To do so, we require an identifying formula for this interventional distribution,
that is, an equation in the observational density that solves for
$f(\mathbf{y}\mid \doop{\mathbf{t}})$
for any $f$ compatible with the causal graph.
We refer to inferring such an identifying formula from a causal graph as inferring the causal effect.
{An effect is identifiable in a causal graph $\g$ if there is at least one identifying formula.}

\paragraph{Valid adjustment.}
Let $\mathbf{T},\mathbf{Y}$, and $\mathbf{Z}$ be pairwise disjoint node sets
in a causal DAG or CPDAG $\g$.  $\mathbf{Z}$ is a valid adjustment set if $f(\mathbf{y}\mid \doop{\mathbf{t}}) = \int f(\mathbf{y}\mid \mathbf{t},\mathbf{z}) f(\mathbf{z}) \,\mathrm{d}\mathbf{z}$ for any density $f$ compatible with $\g$.
Graphical criteria fully characterize valid adjustment sets in DAGs, CPDAGs, and other graph types \citep{perkovic2018complete}.

\paragraph{Causal ordering.}
Let $\g$ be a DAG with node set $\mathbf{V}$. A strict partial order $\prec$ on $\mathbf{V}$ is called a causal order of $\g$ if for all nodes $A,B\in\mathbf{V}$ with $A \rightarrow B$ it holds that $A \prec B$. 
In general, there are multiple causal orders.
For two DAGs $\g$ and $\mathcal{H}$ with node set $\mathbf{V}$, we say that $\g$ respects the causal orders of $\mathcal{H}$ if every causal order of $\g$ is a causal order of $\mathcal{H}$.
We define $\mathrm{pre}_{\prec}(B) = \{A\in\mathbf{V}\mid A\prec B\}$, $\mathrm{post}_{\prec}(A) = \{B\in\mathbf{V}\mid A\prec B\}$, and $\g_\prec$ as the transitively closed DAG with $A\to B$ if and only if $A\prec B$.

\section{Causal Identification Distance}
\label{section: strategic distance def}

We introduce a framework for developing identifiability distances between causal graphs.
This framework lays out how to extend distances to different graph types
and align them with how causal graphs are used to answer causal queries.

\subsection{Framework}

In our framework,
a distance is defined by
a) a sound and complete identification strategy and 
b) a verifier. %
We use the identification strategy to derive identification formulas based on $\Gguess$
and the verifier to evaluate whether the identification formulas obtained on $\Gguess$ are correct in $\Gtrue$. 
Intuitively, the identification strategy represents how an idealized practitioner would use $\Gguess$ to infer causal effects and the verifier evaluates how often the practitioner would be wrong if the ground truth graph were $\Gtrue$.
For simplicity, we only consider single-node interventions
while the framework generalizes when provided a sound and complete identification strategy and verifier (cf.\@ also \cref{sec:osetaid}).

\begin{Definition}[Identification Strategy]
\label{definition: identification strategy}
    An \emph{identification strategy} is an algorithm that for a tuple $(\g,T,Y)$ of a causal graph $\g$ and two distinct nodes $T$ and $Y$ in $\g$, returns the tuple $(T,Y)$ and either an identifying formula $I$ for $f(y\mid \doop{t})$ or $\none$.
An identification strategy $\mathcal{I}$ is sound and complete if a) $\mathcal{I}(\g,T,Y)\neq\none$ if and only if $f(y\mid \doop{t})$ is identifiable in $\g$ and b) all returned identifying formulas are correct for any $f$ compatible with $\g$.
\end{Definition}

\begin{Example}[Parent Adjustment Strategy]
    \label{example: parental adjustment strategy}
    For a DAG $\g$ and two distinct nodes $T$ and $Y$, $\mathbf{P}_T=\pa(T,\g)$ is a valid adjustment set whenever $Y \notin \mathbf{P}_T$. If, on the other hand, $Y \in \mathbf{P}_T$, then, by the acyclicity of $\g$, $Y$ is a non-descendant of $T$ and so there is no causal effect from $T$ on $Y$. We can combine these two results to obtain the sound and complete \emph{parent adjustment strategy}
    \begin{align*}
        \mathcal{I}_{P}(\g,T,Y)= \begin{cases}
            \int f(y\mid t , \mathbf{p}_T) f(\mathbf{p}_T) \, \mathrm{d}\mathbf{p}_T & \text{if } Y \notin \mathbf{P}_T,\\
            f(y) & \text{else.}
        \end{cases}
    \end{align*}
In a DAG, all causal effects are identifiable and therefore $\mathcal{I}_P$ never returns $\none$.
\end{Example}

\begin{Definition}[Verifier]
\label{definition: verification strategy}
    A \emph{verifier} is an algorithm that given a graph $\g$, distinct nodes $T$ and $Y$ in $\g$, and an identifying formula $I$ for $f(y\mid \doop{t})$, verifies whether $I$ is correct for all densities compatible with $\g$
    and returns either $\correct$ or $\false$. 
    For an input of $\none$ it verifies that the effect is not identifiable in $\g$, that is, that no identifying formula exists.
\end{Definition}

Identification has been widely studied; for example, various sufficient conditions for the validity of adjustment sets are known \citep{pearl1993bayesian,maathuis2013generalized}.
Verification, however, has received limited attention
and no algorithm is available to verify an arbitrary identifying formula for DAGs or CPDAGs.
Yet, for some types of identifying formulas, necessary and sufficient graphical criteria exist and these can be used for verification.
In particular, for identifying formulas that use adjustment, we can use the necessary and sufficient
adjustment criterion for verification  \citep{shpitser2010validity,perkovic2018complete}.
A necessary and sufficient criterion also exists for instrumental variables but only for linear models \citep{henckel2023graphical}.
We focus on adjustment-based identification strategies but our framework is amenable to other strategies provided a corresponding verifier exists.

\begin{Example}[Adjustment-Verifier for DAGs]
\label{example: DAG adjustmnet verifier}
    The identification strategy $\mathcal{I}_P$ relies on two identification principles: a) valid adjustment and b) non-descent.
    The verifier $\mathcal{V}_{\text{adj}}$ in \cref{algorithm: DAG adjustment verifier} is simple, sound, and complete
    and uses the adjustment criterion and a non-descent check
    for the verification of adjustment-based identification strategies. %
    In DAGs all effects are identifiable
    and so
    the verifier rejects any $\none$-identification formula as $\false$,
    that is,
    the if-branch in line 5 is not reached
    and only included for completeness.
\end{Example}

\begin{Definition}[$\mathcal{I}$-Specific Identification Distance]
\label{definition: strategy specific distance}
    Given a sound and complete identification strategy $\mathcal{I}$,
    we define the \emph{$\mathcal{I}$-specific identification distance} between two graphs $\Gtrue$ and $\Gguess$ with
    common node set $\mathbf{V}$ as the number of identification 
    formulas
    inferred by 
    $\mathcal{I}$ on 
    $\Gguess$ that are incorrect relative to $\Gtrue$, that is,
\begin{multline*}
d^{\mathcal{I}}(\Gtrue,\Gguess,\mathbf{S}) \\
=
\sum_{(T, Y) \in \mathbf{S}}
\mathbb{1}_{\{\false\}}\big(\
\mathcal{V}\left(
\Gtrue,
\mathcal{I}(\Gguess,T,Y)
\right)\big)
\end{multline*}
    where $\mathbf{S}\subseteq \overline{\mathbf{S}}=\{ (T,Y)\in \mathbf{V}\times\mathbf{V} \mid T\neq Y \}$
    and $\mathcal{V}$ is a verifier.
    Unless otherwise noted, we
    use all $p(p-1)$ pairs of distinct nodes in $\mathbf{V}$
    and write $d^{\mathcal{I}}(\Gtrue,\Gguess)=d^{\mathcal{I}}(\Gtrue,\Gguess,\overline{\mathbf{S}})$.\footnote{%
    The flexibility to choose other sets $\mathbf{S'}\subset \overline{\mathbf{S}}$
    allows one to tailor the distance $d^\mathcal{I}(\Gtrue,\Gguess,\mathbf{S'})$
    to consider only some specific nodes of interest as treatment or effect nodes
    or, given a suitable identification strategy and verifier,
    to consider multi-node interventions
    when comparing graphs
    (see also \cref{sec:osetaid}).
    }
\end{Definition}

Formally, the $\mathcal{I}$-specific identification distance between $\Gtrue$ and $\Gguess$ counts how many of the identification formulas obtained by using the identification strategy $\mathcal{I}$ on $\Gguess$ are incorrect for $\Gtrue$. Intuitively, the distance is the number of causal effects an idealized practitioner would wrongly infer from $\Gguess$, if $\mathcal{I}$ resembled how the practitioner would use the graph $\Gguess$ for causal inference while the true DAG were $\Gtrue$. By construction, strategy-specific identification distances are asymmetric in their input graphs.

\begin{Example}[SID is the $\mathcal{I}_{P}$-Specific Distance for DAGs]
Let $\Gtrue$ and $\Gguess$ be  DAGs with common node set $\mathbf{V}$.
Then $d^{\mathcal{I}_P}(\Gtrue,\Gguess)$ coincides with $\operatorname{SID}(\Gtrue,\Gguess)$ as defined by \citet{peters2015structural}.
For CPDAGs, however, the SID is defined as the multi-set obtained by calculating the SID for each DAG in the Markov equivalence class
and is not a distance in our framework;
in \cref{sec:cpdagdistances} we present the canonical extension of the $\mathcal{I}_P$-specific distance to CPDAGs that outputs a scalar and retains interpretability.
\label{example: SID}
\end{Example}

\begin{algorithm}[t]
\caption{Adjustment-Verifier $\mathcal{V}_{\text{adj}}$}\label{algorithm: DAG adjustment verifier}
\begin{algorithmic}[1]
\State \textbf{Input}: Graph $\g$, tuple $(T,Y)$, identifying formula $I$
\State \textbf{Output}: Validity indicator $V\in\{\correct,\false\}$
\vspace{.25em}
\State $V \gets \false$
\If{$I=\none$ and $f(y\mid \doop{t})$ not identifiable in $\g$} 
\State $V \gets \correct$ 
\ElsIf{$I=f(y)$ and $Y \in \mathrm{NonDe}(T,\g)$} 
\State $V \gets \correct$ 
\ElsIf{$I=\int f(y \mid t,\mathbf{z}) f(\mathbf{z}) \ \,\mathrm{d}\mathbf{z}$ and $\mathbf{Z}$ is a valid adjustment set relative to $(T,Y)$ in $\g$}\State $V \gets \correct$
\EndIf
\vspace{.25em}
\State \Return{V}
\end{algorithmic}
\end{algorithm}

\subsection{General Properties}
\label{section: strategic distance properties}

Any identification distance between a DAG and a supergraph of that DAG is zero.
A corollary
highlights that identification distances
differ from the SHD:
there exist DAGs for which identification distances 
are maximally different from the SHD.
Proofs are provided in \cref{supp:proofs}.

\begin{Proposition}[Distance to Super-DAG is Zero]
    Let $\mathcal{I}$ be a sound and complete identification strategy for DAGs. %
    For any DAG $\Gtrue$, it holds that
    if $\Gguess$ is a super-DAG of $\Gtrue$,
    then    
    $d^{\mathcal{I}}(\Gtrue,\Gguess)=0$.
    \label{prop: supergraph distance}
\end{Proposition}

As 
$d^{\mathcal{I}}(\Gtrue,\Gtrue)=0$,
identification distances
are pre-metrics.
\cref{prop: supergraph distance} is a consequence of all causal effects in a DAG being identifiable; %
adding edges removes the information that certain effects are absent
and may reduce the number of correct identifying formulas, but never to zero.
For other graph types such as CPDAGs, this is not the case (cf.\@ \cref{sec:cpdagdistances}).
As a corollary,
DAGs may be close in
identification distance
but far in SHD.

\begin{Corollary}[Identification Distances Differ from the SHD]
    \label{corollary: hamming}
    Let $d^{\mathcal{I}}$ be a strategy-specific identification distance. Let $\Gguess$ be a fully connected DAG with $p$ nodes and $\Gtrue$ the empty DAG on the same node set.
    Then
    the SHD $d_{H}$ is maximal
    and maximally different from
    $d^\mathcal{I}$:
    \begin{equation*}
        d_{H}(\Gtrue,\Gguess) -d^{\mathcal{I}}(\Gtrue,\Gguess)=p(p-1)/2 - 0.
    \end{equation*}
\end{Corollary}

\section{DAG Distances}
\label{section: DAG distances}

We propose three adjustment-based distances for DAGs and extend them to CPDAGs in \cref{sec:cpdagdistances}.
We propose a \emph{parent adjustment distance} which between DAGs corresponds to the SID but, in contrast to the SID, generalizes canonically to CPDAGs.
We develop an \emph{ancestor adjustment distance}
that
assigns low distance to graphs with similar causal orders and an \emph{Oset adjustment distance}
that uses a statistically efficient identification strategy. 
We discuss how each distance corresponds to different assumptions on how a practitioner would use a graph for causal inference.
Users need to choose (a combination of) distances
based on how they envision the graph will be used for causal reasoning in a downstream task.
Depending on the downstream task, even the SHD may be considered a causal graph distance, for example, if the downstream task merely involves reasoning about the existence of direct cause-effect relationships but not the identification or estimation of those direct effects.

\subsection{Parent Adjustment Distance}

We call $d^{\mathcal{I}_P}$ the \emph{parent adjustment distance (Parent-AID)}.
The Parent-AID is an identification distance (cf.\@ \cref{example: SID}) but yields unintuitive results between graphs with the same causal orders (cf.\@ \cref{lemma: causal order})
and uses an inefficient adjustment strategy (cf.\@ \cref{sec:osetaid});
we include it for completeness,
as it includes the SID for DAGs as a special case, but, in contrast to the SID, canonically extends to CPDAGs within our framework (cf.\@ \cref{sec:cpdagdistances}).
We develop refined adjustment-based distances in the next two subsections.

Parent adjustment is used in practice \citep{gascon2015prenatal,sunyer2015association}. %
To reason about the effect of intervening on a variable $T$,
it requires one to know only the direct causes of $T$ but not the full causal graph;
if a practitioner knew the parent sets for all nodes infact they would know the full causal DAG.
The Parent-AID assumes a practitioner who follows this common practice.
The parent adjustment strategy is local,
that is, for any pair $(T,Y)$ the adjustment set only depends on $T$ but not $Y$.
We use this to improve the time complexity of calculating the distance (cf.\@ \cref{section: implementation}).
For DAGs, $d^{\mathcal{I}_P}(\Gtrue,\Gguess)=0$ if and only if $\Gguess$ is a supergraph of $\Gtrue$ \citep{peters2015structural}.

\subsection{Ancestor Adjustment Distance}

Many causal discovery algorithms learn an ordering of the nodes in a separate first step \citep{shojaie2010penalized,buhlmann2014cam,chen2019causal,park2020identifiability}
and pairwise causal effects can be learned from the correct causal order alone \citep[Section 2.6]{buhlmann2014cam}.
We formalize this known result as follows.

\begin{Lemma}[Ancestors are Valid Adjustment Sets]
\label{lemma: non-descendants}
    Let $T$ and $Y$ be two distinct nodes in a DAG $\g$.
    Then any set $\mathbf{Z}$ such that $\pa(T,\g) \subseteq \mathbf{Z} \subseteq \mathrm{NonDe}(T,\g)$ and $Y\notin \mathbf{Z}$ is a valid adjustment set for $(T,Y)$ in $\g$. As a corollary, given an order $\prec$, $\mathrm{pre}_{\prec}(T)$ is a valid adjustment set for all $Y \notin \mathrm{pre}_{\prec}(T)$ in all DAGs for which $\prec$ is a causal order.
\end{Lemma}

Thus, it is possible to derive identification formulas in $\Gguess$ that are also correct in $\Gtrue$,
if $\Gguess$ respects the causal orders of $\Gtrue$.
Yet, even if two DAGs have the same causal orders, the Parent-AID between them can be large.

\begin{Lemma}[Parent-AID Misrepresents Causal Order]
\label{example: causal order}
\label{lemma: causal order}
Let $\Gtrue^p$ be the fully connected DAG over $p$ causally-ordered nodes
$\{V_1,...,V_p\}$
and $\Gguess^p$ the chain $V_1\to V_2\cdots\to V_p$
(cf.\@ \cref{figure: parental adjustment example}).
Then,
despite $\Gtrue$ and $\Gguess$ respecting each others causal orders, $d^{\mathcal{I}_{P}}(\Gtrue^p,\Gguess^p) = p^2-4p+4$
which is close to its maximal value $p(p-1)$
in the sense that
\begin{equation*}
\lim_{p \to \infty} {d^{\mathcal{I}_{P}}(\Gtrue^p,\Gguess^p)}/{p(p-1)} = 1.
\end{equation*}
\end{Lemma}

Considering DAGs as distant that respect each others causal orders may be unintuitive.
We therefore propose an alternative adjustment strategy and distance.

\begin{Definition}[Ancestor Adjustment Strategy and Distance]
\label{definition: ancestral distance}
Given two distinct nodes $T$ and $Y$ in a DAG $\g$, let $\mathbf{A}_T=\mathrm{An}(T,\g) \setminus \{T\}$ and $\mathbf{D}_T=\de(T,\g)$. We define the \emph{ancestor adjustment strategy} as
\begin{align*}
        \mathcal{I}_{A}(\g,T,Y)= \begin{cases}
            \int f(y\mid t , \mathbf{a}_T) f(\mathbf{a}_T) \,\mathrm{d}\mathbf{a}_T & \text{ if } Y \in \mathbf{D}_T, \\ %
            f(y) & \text{ else,}
        \end{cases}%
    \end{align*}
which is a sound and complete identification strategy per \cref{prop: ancestral adjustment}.
We call the corresponding distance $d^{\mathcal{I}_A}$ the \emph{ancestor adjustment distance (Ancestor-AID)}.
\end{Definition}

\begin{figure}[t!]
\centering
   \begin{tikzpicture}[scale=1]
      \node[state] (v1) at (0,0) {$V_1$};
      \node[state] (v2) at (1.5,0) {$V_2$};
      \node[state] (v3) at (3,0) {$V_3$};
      \node[state] (v4) at (4.5,0) {$\dots$};
      \node[state] (v5) at (6,0) {$V_{p}$};
      \path (v1) edge[directed] (v2)
            (v2) edge[directed] (v3)
            (v3) edge[directed] (v4)
            (v4) edge[directed] (v5)
            
            (v1) edge[directed, bend left=20, dashed] (v3)
            (v1) edge[directed, bend left=20, dashed] (v4)
            (v1) edge[directed, bend left=20, dashed] (v5)
            (v2) edge[directed, bend right=20, dashed] (v4)
            (v2) edge[directed, bend right=20, dashed] (v5)
            (v3) edge[directed, bend right=20, dashed] (v5);
    \end{tikzpicture}
\vspace{-.2cm}
\caption{Fully connected and chain DAG in \cref{example: causal order}.}
    \label{figure: parental adjustment example}
\end{figure}

\begin{Proposition}[Ancestor-AID Reflects Causal Order]
    The ancestor adjustment strategy $\mathcal{I}_{A}$ is sound and complete for DAGs
    and so the corresponding ancestor adjustment distance $d^{\mathcal{I}_A}$ is the $\mathcal{I}_A$-specific identification distance.
    Further, for any two DAGs $\Gtrue$ and $\Gguess$ with the same node set, %
        $d^{\mathcal{I}_{A}}(\Gtrue,\Gguess) = 0$
    if and only if $\Gguess$ respects the causal orders of $\Gtrue$.%
    \label{prop: ancestral adjustment}
\end{Proposition}

Due to \cref{prop: ancestral adjustment}, the Ancestor-AID is preferable to the Parent-AID for evaluating the causal order of a learned graph. It can be used both as a replacement and to complement the Parent-AID.

The ancestor adjustment strategy is local, that is, for any pair $(T,Y)$ the adjustment set only depends on $T$ but not $Y$.
Adjusting for the ancestors has some advocates \citep{rubin2008design} and is at least as statistically efficient as parent adjustment \citep{henckel2019graphical}. 
A practitioner more confident in the ability of causal discovery algorithms to learn causal orders rather than all specific edges and exact parent sets, may prefer the ancestor adjustment strategy over the parent adjustment strategy to infer causal effects from a learned graph.

\subsection{Oset Adjustment Distance}\label{sec:osetaid}

In practice, identifying formulas are a tool to estimate a causal effect of interest.
Different identifying formulas correspond to different estimators for this effect.
For example, in linear models we can estimate the average treatment effect with an ordinary least squares regression of $Y$ on $T$ and $\mathbf{Z}$;
this estimator is consistent for any valid adjustment set $\mathbf{Z}$.
Other properties of the estimator, such as its asymptotic variance, however, depend on the adjustment set.

We can use the causal graph to decide which valid adjustment sets result in statistically efficient estimators;
importantly, for a large class of estimators, the valid adjustment set $\pa(T,\g)$ is close to the least efficient among all valid adjustment sets
\citep{rotnitzky2020efficient,witte2020efficient,henckel2019graphical}. 
The parents are therefore an inefficient adjustment set
and parent adjustment is perhaps not good practice.
As an alternative \citet{henckel2019graphical} have proposed the optimal adjustment set. 

\begin{Definition}[Optimal Adjustment Set (Oset)]
    Let $T$ and $Y$ be two distinct nodes in a DAG $\g$. Then the \emph{optimal adjustment set (Oset)} $\opt{\g}$ is defined as
    \begin{equation*}
        \opt{\g} = \pa(\cn{\g},\g) \setminus \forb(T,Y,\g)
    \end{equation*}
where $\cn{\g}$ are the causal and $\forb(T,Y,\g)$ the forbidden nodes as defined in \cref{supp:preliminaries}.
\end{Definition}

If $Y \in \de(T,\g)$, then $\opt{\g}$ is a valid adjustment set whenever a valid adjustment set exists.
For a large class of estimators, the Oset is the most statistically efficient among all valid adjustment sets.
We use this result to propose another adjustment-based identification distance.

\begin{Definition}[Oset Adjustment Strategy and Distance]
\label{definition: optimal distance}
Given two distinct nodes $T$ and $Y$ in a DAG $\g$, let $\mathbf{O}_T=\opt{\g}$ and $\mathbf{D}_T=\de(T,\g)$. We define the \emph{Oset adjustment strategy} as
\begin{align*}
        \mathcal{I}_{O}(\g,T,Y)= \begin{cases}
            \int f(y\mid t , \mathbf{o}_T) f(\mathbf{o}_T) \,\mathrm{d}\mathbf{o}_T & \text{ if } Y \in \mathbf{D}_T, \\ %
            f(y) & \text{ else,}
        \end{cases}
    \end{align*}
which is a sound and complete identification strategy per \cref{prop: optimal adjustment DAGs}.
We call the corresponding 
distance $d^{\mathcal{I}_O}$ the \emph{Oset adjustment distance (Oset-AID)}.
\end{Definition}

\begin{Proposition}[Oset-AID is the $\mathcal{I}_O$-Specific Distance]
    \label{prop: optimal adjustment DAGs}
    The Oset adjustment strategy $\mathcal{I}_{O}$ is sound and complete for DAGs %
    and so the corresponding Oset adjustment distance $d^{\mathcal{I}_{O}}$ is the $\mathcal{I}_O$-specific identification distance.%
\end{Proposition}

Oset adjustment has seen some early adoption by practitioners \citep{steiger2021causal}.
Given the Oset's efficiency guarantee, the Oset-AID assumes that a practitioner takes efficiency into account when selecting valid adjustment sets.
The Oset adjustment strategy
is non-local as the Oset
depends on both $T$ and $Y$.
As a result, the Oset adjustment distance is computationally expensive with polynomial complexity of one order higher than that of the Parent- and Ancestor-AID which use local adjustment strategies (cf.\@ \cref{section: implementation}).
Further, we do not have a graphical characterization of all cases where the Oset adjustment distance is zero
(cf.\@ \cref{example: zero optimal distance} in \cref{section: additional examples}).

\paragraph{Joint interventions.}
Another advantage of the Oset adjustment strategy over the parent or ancestor adjustment strategies is that
the Oset---in contrast to the parents or ancestors---is a valid adjustment set whenever a valid adjustment set exists,
even if we consider joint interventions.
As such, the Oset adjustment strategy enables generalizations of the Oset adjustment distance to settings where $\mathbf{S}$
may contain multi-node interventions.
However, adjustment is not sound and complete for effects of joint interventions
and there may exist identifiable effects that are not identifiable via adjustment.
Therefore, this generalization is strictly speaking not a strategy-specific identification distance.
Identifiable joint intervention effects that cannot be identified via adjustment are characterized in Corollary 27 of \citet{perkovic2018complete} and can be correctly identified by other strategies \citep{nandy2017estimating,huang2006identifiability}.
However, there is no verifier for these alternative strategies.
More research into verification is required to develop a proper strategy-specific distance that considers joint interventions.

\section{CPDAG Distances}\label{sec:cpdagdistances}

In general and without strong assumptions, the true causal DAG cannot be identified or
learned, even from infinite data \citep[e.g.\@][]{peters2014causal}. 
Instead,
many causal discovery algorithms target the corresponding Markov equivalence class and aim to learn its CPDAG \citep[e.g.][]{Chickering02}.
To evaluate common causal discovery algorithms,
we thus also need easy to compute and interpret distances
between CPDAGs.
Our strategy-specific distance framework
provides a recipe on how to develop such distances:
a) devise a sound and complete identification strategy and b) devise a corresponding verifier.
We follow this recipe to propose new and computationally attractive distances for CPDAGs based on the parent, ancestor, and Oset adjustment strategies.
The distances operate directly on the CPDAGs, assess the compatibility of identification formulas between the two graphs, and return a scalar distance.
This improves upon previous approaches that iterate over an exponentially large number of Markov equivalent DAGs
to calculate a multi-set of distances, which is computationally prohibitive and difficult to interpret.
We also discuss potential distances across graph types and their pitfalls.

\subsection{CPDAG to CPDAG Distances}

In contrast to DAGs, not all causal effects are identifiable given a CPDAG \citep{meek1995causal}.
For example, we cannot identify the effect of $A$ on $B$ given the CPDAG $A\arrow B$. 
To extend the SID to CPDAGs, \citet{peters2015structural} combine two approaches. 
The first, is to consider all DAGs in the Markov equivalence class of the CPDAG $\Gguess$ and compute a multi-set of distances using the SID for DAGs.
The second, is to ignore all $(T,Y)$ node tuples for which the effect of $T$ on $Y$ is not identifiable in $\Gtrue$.
In principle, the first approach could also be applied to $\Gtrue$ and the second to $\Gguess$.
Either way, both approaches of extending the SID to CPDAGs have drawbacks.

The first approach yields a difficult to interpret multi-set of values and is in general computationally infeasible beyond graphs with very few nodes and small Markov equivalence class. For example,
let
$\Gtrue$ be the empty CPDAG (which is also a DAG)
and
$\Gguess$ the fully connected CPDAG. 
The SID between the empty and the fully connected CPDAG is a multi-set of $0$s but to compute it one needs to iterate over all fully connected DAGs.
Arguably, the two CPDAGs are also maximally different when used to infer causal effects, since all effects are identifiable in the empty CPDAG but no effect is identifiable in the fully connected CPDAG.
The second approach discards valuable information. For example, if no effect is identifiable in $\Gtrue$, then the SID between $\Gtrue$ and any other CPDAG $\Gguess$ is a multi-set of $0$s.
The SID between CPDAGs as proposed by \citet{peters2015structural} inherits these drawbacks\footnote{In fact, the CRAN SID package v1.1 requires the true graph to be a DAG and does not implement the distance between two CPDAGs outlined by \citet{peters2015structural}.}
and in contrast to the SID between DAGs, is not an easy-to-interpret identification distance within our framework.
 
Our framework offers an alternative. Indeed, following our framework of identification distances there is a canonical solution to developing distances between CPDAGs:
have identification strategies return $\none$ in case an effect is non-identifiable in $\Gguess$ and
treat non-identifiability as a claim that we can verify in $\Gtrue$ just like we verify identifying formulas returned by the identification strategy.
Identification distances always return a single scalar value and – by sidestepping iteration over exponentially large Markov equivalence classes and using efficient verification algorithms on CPDAGs – they are computationally tractable even for large CPDAGs.
Furthermore, as identification distances they are interpretable since they capture how often a practitioner would wrongly infer a causal effect when using the learned instead of the true CPDAG.
Importantly, we do not presuppose that a practitioner would (randomly) pick a DAG within the Markov equivalence class of the learned CPDAG and just use that DAG for causal inference;
instead, we posit a practitioner would only infer those effects that are identifiable in the learned CPDAG and
would rather look into learning a more refined graph than attempting to reason about non-identifiable effects.

As a result, however, there is no clear relationship of the strategy-specific identification distances between two DAGs, to the distances between the two corresponding CPDAGs. For example, the DAGs $V_1 \rightarrow \dots\rightarrow V_p$ and $V_1 \leftarrow \cdots \leftarrow V_p$ have the same CPDAG and yet the adjustment distance between the two DAGs is maximal (irrespective of the strategy).
In contrast to DAGs (cf.\@ \cref{section: strategic distance properties}), adding edges in CPDAGs may render some causal effects non-identifiable and no analogous statement to \cref{prop: supergraph distance} holds for identification distances between CPDAGs; indeed, since all effects are identifiable in the empty CPDAG but none in the fully connected CPDAG, their identification distance is maximal.
Further, the distance between any two CPDAGs in which no effect is identifiable is zero. For example, if $\Gtrue$ and $\Gguess$ are both CPDAGs consisting of a single undirected path connecting all nodes, then their distance is zero even though they may not have a single edge in common.
While this behavior is less extreme than for the SID where the distance between
a true CPDAG with no identifiable effects and
any learned CPDAG
is zero, it may nonetheless seem unintutive.
However, following the interpretation of identification distances, the distance of zero between any two CPDAGs in which no effect is identifiable is reasonable since a practitioner given such a $\Gguess$ would conclude that no effect is identifiable as is indeed the case in $\Gtrue$.

In CPDAGs, non-identifiability is characterized 
by a graphical condition called amenability by \citet{perkovic2018complete,perkovic2020identifying},
which for single-node interventions is as follows.

\begin{Proposition}[Amenability]
Consider distinct nodes $T$ and $Y$ in a CPDAG $\g$. The interventional density $f(y\mid \doop{t})$ is identifiable if and only if there exists no possibly directed path from $T$ to $Y$ that starts with an undirected edge. If this holds, we say that $(\g,T,Y)$ is amenable.
\end{Proposition}

Equipped with graphical conditions for non-identifiability and validity of adjustment sets in CPDAGs,
we can apply the adjustment verifier in \cref{algorithm: DAG adjustment verifier} to CPDAGs.
Within our framework, any sound and complete adjustment-based identification strategy together with this verifier defines a strategy-specific identification distance for CPDAGs.
To extend the above adjustment distances to CPDAGs,
we need to extend the identification strategies by adding an amenability check
such that
$\none$ is returned for $(\g,T,Y)$ that are not amenable
and else the return values of the identification strategies $\mathcal{I}_P$, $\mathcal{I}_A$, and $\mathcal{I}_O$ are the same as for DAGs.

The parent adjustment strategy is sound and complete for CPDAGs
\cite[Corollary 4.2]{maathuis2013generalized},
as are the Oset \citep[Theorem 3]{henckel2019graphical}
and ancestor adjustment strategies (\cref{proposition: ancestral CPDAGs}, \cref{supp:proofs}).
Thus, we generalize
the Parent-AID, the Ancestor-AID, and the Oset-AID
to distances between CPDAGs.

\subsection{DAG, CPDAG, and Order Distances}

\paragraph{DAG to CPDAG distance.}
Given a suitable identification strategy and verifier,
we can define a strategy-specific distance between graphs of different type.
For example, since the presented adjustment strategies and verifier apply to both DAGs and CPDAGs,
our distances accept any combination of DAG and CPDAG as $\Gtrue$ and $\Gguess$.
Yet, such a distance may be unintuitive:
The distance between a DAG $\Gtrue$
and the CPDAG $\Gguess$ that encodes the Markov equivalence class of $\Gtrue$
is generally non-zero
as some effects are non-identifiable in the CPDAG;
the distance to this correct CPDAG may in fact be further than to another CPDAG that encodes a Markov equivalence class that does not contain $\Gtrue$ (cf.\@ \cref{example: DAG to CPDAG}, \cref{section: additional examples}).
A cross-graph-type distance may still be useful when comparing 
an algorithm that can learn a DAG, such as LiNGAM \citep{shimizu2006linear}, to an algorithm that cannot, such as GES \citep{Chickering02}.
Our implementation therefore accepts DAG-to-CPDAG and CPDAG-to-DAG comparisons.

\paragraph{Transformations to compare alike.}
An alternative is
to transform one graph type to the other and then apply a distance between DAGs or between CPDAGs.
To obtain a proxy for the distance between a DAG $\Gtrue$ and a CPDAG $\Gguess$,
for example,
one could pick a DAG corresponding to $\Gguess$ and compare that to $\Gtrue$;
common approaches are
a) to sample a DAG in the Markov equivalence class of $\Gguess$ or
b) to orient undirected edges in $\Gguess$ for which a corresponding edge in DAG $\Gtrue$ exists correctly and the remaining undirected edges randomly while ensuring acyclicity.
Both approaches are ad-hoc, non-deterministic, and ignore causal information in the CPDAG $\Gguess$, such as
claims about which effects are not identifiable. 
Our CPDAG distance enables a principled alternative: transform the DAG $\Gtrue$ to its corresponding CPDAG and then compare the CPDAG corresponding to the true DAG to the learned CPDAG.
This approach is natural, when 
the test data is simulated according to a DAG $\Gtrue$
    and we compare the performance of two CPDAG learning algorithms, such as GES and PCALG\footnote{In finite samples the output of PCALG may not be a CPDAG and for these graphs no identification strategy exists. As a result, it may be necessary to resolve PCALG conflicts or to use non-causal distances when evaluating the performance of PCALG \citep{wahl2024metrics}.} \citep{Chickering02,spirtes2000causation}.

\paragraph{DAG to order distance.} 
Given a strict partial order $\prec$ on nodes $\mathbf{V}$ we can define the identification strategy 
\begin{align*}
        \mathcal{I}_\text{ord}(\prec,T,Y)= \begin{cases}
            \int f(y\mid t , \mathbf{b}_T) f(\mathbf{b}_T) \,\mathrm{d}\mathbf{b}_T & \text{ if } Y \in \mathbf{A}_T, \\ %
            f(y) & \text{ else,}
        \end{cases}%
    \end{align*}
where $\mathbf{B}_T= \mathrm{pre}_{\prec}(T)$ and $\mathbf{A}_T= \mathrm{post}_{\prec}(T)$. By \cref{lemma: non-descendants} this strategy is sound and complete and we can verify the returned identification formulas in a DAG $\Gtrue$ using $\mathcal{V}_{\text{adj}}$. As such we obtain a strategy-specific distance $d^{\mathcal{I}_\text{ord}}(\Gtrue,\prec_{\mathrm{guess}})=d^{\mathcal{I}_A}(\Gtrue, \g_{\prec_\text{guess}})$ between DAGs and strict partial orders.
The strategy-specific distance $d^{\mathcal{I}_\text{ord}}(\Gtrue,\prec_{\mathrm{guess}})$ counts the 
number of identification formulas derived from the partial order $\prec_{\mathrm{guess}}$ that would be wrong if the true graph were $\Gtrue$. It allows for direct comparison between a learned causal order and the true DAG. It offers an alternative to existing approaches, such as rank correlations or order distances that lower bound the SHD \citep{rolland2022score} and that may be difficult to interpret or computationally expensive because in general the causal order of a graph is neither unique nor a total order.

\section{Implementation}
\label{section: implementation}

\begin{figure*}
{\centering
\includegraphics[draft=false]{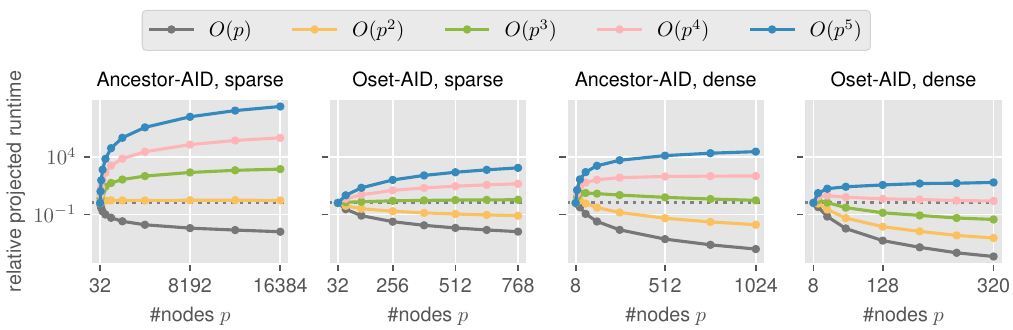}
}
\definecolor{Osecond}{rgb}{0.984314,0.756863,0.368627}%
\definecolor{Othird}{rgb}{0.556863,0.729412,0.258824}%
\definecolor{Ofourth}{rgb}{1.000000,0.709804,0.721569}%
\caption{Empirical results on the algorithmic time complexity of calculating the Ancestor-AID $d^{\mathcal{I}_A}$ and the Oset-AID $d^{\mathcal{I}_O}$ between random sparse
and dense
graphs.
    We project the runtime under the different time complexities based on the smallest graphs in each panel
    and visualize the projected runtime as a fraction of the observed empirical runtime;
    if the relative projected runtime increases/decreases with increasing number of nodes,
    the considered time complexity suggests a faster/slower increase of runtime than empirically observed.
The empirical analysis suggests that our implementation of the Ancestor-AID achieves the time complexity of {\colorbox{Osecond}{$O(p^2)$}} for sparse and {\colorbox{Othird}{$O(p^3)$}} for dense graphs,
and that the implementation of the Oset-AID
achieves the time complexity of {\colorbox{Othird}{$O(p^3)$}} for sparse and {\colorbox{Ofourth}{$O(p^4)$}} for dense graphs.
See \cref{app:complexityexperiment} for details.
}
\label{fig:runtime}
\end{figure*}

We sketch our implementation of the distances for CPDAG inputs with $p$ nodes and $m$ edges;
see
\cref{section: appendix implementation}
for details.
First, consider
a single tuple $(T,Y)$.
For the identification strategy, we check whether $(\Gguess,T,Y)$ is amenable
and if so
compute
a) $\pa(T,\Gguess)$ for the Parent-AID,
b)
$\de(T,\Gguess)$ and $\an(T,\Gguess)$ for the Ancestor-AID,
or  c)
$\de(T,\Gguess)$ and $\opt{\Gguess}$\footnote{%
In \cref{lemma: alternative optimal}, \cref{section: additional examples},
we prove a characterization of the Oset that, given amenability,
simplifies its computation.
}
for the Oset-AID.
For the verifier, we check
whether a) $(\Gtrue,T,Y)$ is amenable, b) $Y \in \mathrm{NonDe}(T,\Gtrue)$, or c) the proposed adjustment set 
is a valid adjustment set for $(T,Y)$ in $\Gtrue$.
Algorithms exist to perform each of these steps in $O(p+m)$ time
\citep{vanconstructing},
so we can compute the distances in $O(p^2(p+m))$ time
by iterating over all tuples.

We
improve this complexity
for the Parent- and Ancestor-AID
by sharing computations between tuples instead of evaluating identification strategy and verifier for each of the $p(p-1)$ tuples separately.
For this we use
reachability algorithms, which
are inspired by the Bayes-Ball algorithm \citep{geiger1989dseparation,shachter1998bayes}:
They start from a node set, walk along edges per fixed rules,
and return the set of all reached nodes \citep[cf.\@][Appendix C]{wienobst2022finding}.
Reachability algorithms find all nodes with a certain property that depends on the rules used.
For example, given $T$ and the rule to continue only along $\rightarrow$ edges, the search algorithm finds all nodes in $\de(T,\g)$ in $O(p+m)$ time.
There are reachability algorithms to compute $\de(T,\g)$, $\an(T,\g)$, and similar sets. 
We develop new walk-status-aware reachability algorithms
that, given a graph $\g$, treatment $T$, and candidate adjustment set $\mathbf{Z}$,
return a) all nodes such that $(\g,T,Y)$ is amenable,
or b) all nodes $Y$ such that $\mathbf{Z}$ is a valid adjustment set for $(T,Y)$ in $\g$
(\cref{algorithm: amenability,algorithm: VAS check}).

These reachability algorithms
enable our computationally efficient implementation. 
When using local adjustment strategies,
we can fix a $T$ and compute both the identification strategy
and verifier for all $Y$ via at most six reachability algorithms.
Selecting each node as $T$ once,
we can calculate the Parent- and Ancestor-AID in $O(p(p+m))$ time;
this amounts to $O(p^2)$, the optimum, for sparse graphs with $m\in O(p)$
and to $O(p^3)$ for dense graphs.
The asymptotic
runtime complexity of the Oset-AID remains $O(p^2(p+m))$
since $\opt{\g}$ depends on both $T$ and $Y$.

The original SID implementation has $O(p^4\log(p))$ runtime
for DAGs and exponential runtime for CPDAGs \citep{peters2015structural}.
Our implementation of the related Parent-AID
between either DAGs or CPDAGs
has runtime $O(p^3)$ for dense and $O(p^2)$ for sparse graphs.
To our knowledge, our distances are the first causal distances between CPDAGs with a polynomial runtime guarantee.

\section{Empirical Results}

We provide a simulation study quantifying the empirical runtime of our algorithms.
In an additional simulation study, we compare the three distances we propose in this paper across various pairs of graphs.

\subsection{Empirical Runtime Analysis}

\label{sec: empirical runtime}

We calculate
distances
with
our
\texttt{gadjid} 
package
version 0.1.0,
implemented in Rust
and using a graph memory layout purposefully designed for fast memory access in reachability algorithms.
We use the CRAN SID package v1.1 and run all experiments on a laptop
with 8 GB RAM and 4-core i5-8365U processor.
We draw DAGs with $p$ nodes, uniformly random total order of nodes,
and edges compatible with this order independently drawn with probability
$20/(p-1)$ for sparse graphs with $10p$ edges in expectation and
$0.3$ for dense graphs with $0.3p(p-1)/2$ edges in expectation.

To empirically validate the theoretical asymptotic runtime complexities,
we evaluate the Ancestor-AID and the Oset-AID on random DAGs.
For each graph size,
we record the runtime averaged over $5$ repetitions.
Based on the runtimes for the smallest graphs,
we project what runtimes we would expect for larger graphs under various time complexities.
\cref{fig:runtime} shows the results
and \cref{app:complexityexperiment} provides details.

Next, we draw $11$ pairs of random DAGs,
calculate a distance,
and if the median runtime is less than $60$ seconds, we increase the number of nodes by one and repeat;
we repeat until the median runtime exceeds $60$ seconds and obtain:
\begin{center} %
\begin{tabular}{llrr}
 \multicolumn{4}{l}{Maximum graph size feasible within 1 minute}\\
\multicolumn{2}{l}{Method}\hspace{2.9cm}                & sparse\hspace{.6cm}                                                                                 & dense                                                                               \\
\midrule
\multicolumn{2}{l}{Parent-AID}            & \texttt{\BudgetNodesParentsparse}\hspace{.6cm}                                                                                        & \texttt{\BudgetNodesParentdense}                 \\
\multicolumn{2}{l}{Ancestor-AID}          & \texttt{\BudgetNodesAncestorsparse}\hspace{.6cm}                                                                                      & \texttt{\BudgetNodesAncestordense}                                                            \\
\multicolumn{2}{l}{Oset-AID}              & \texttt{\BudgetNodesOsetsparse}\hspace{.6cm}                                                                                          & \texttt{\BudgetNodesOsetdense}                                                               \\
\multicolumn{2}{l}{SID}                 & \texttt{\BudgetNodesRSIDsparse}\hspace{.6cm}                                                                                          & \texttt{\BudgetNodesRSIDdense}                                                                \\
\end{tabular}
\end{center} %

\noindent Finally, we consider the graph sizes for which the average runtime of the SID first exceeded one minute, and the extremely sparse graphs from \citet{peters2015structural}; for $11$ random pairs of graphs of that size and sparsity,
we obtain the following average runtimes:

\begin{center} %
\begin{tabular}{lrrr}
\multicolumn{4}{l}{Average runtime}   \\
Method       & \begin{tabular}[c]{@{}l@{}}x-sparse\footnotemark\\ $p=1000$\end{tabular} & \begin{tabular}[c]{@{}l@{}}sparse\\ $p=\BudgetNodesRSIDsparse{}$\end{tabular} & \begin{tabular}[c]{@{}l@{}}dense\\ $p=\BudgetNodesRSIDdense{}$\end{tabular} \\
\midrule
Parent-AID   & \texttt{\ThousandMeanParent}           & \texttt{\RuntimeMeanParentsparse{}}     & \texttt{\RuntimeMeanParentdense{}}     \\
Ancestor-AID & \texttt{\ThousandMeanAncestor}         & \texttt{\RuntimeMeanAncestorsparse} & \texttt{\RuntimeMeanAncestordense} \\
Oset-AID     & \texttt{\ThousandMeanOset}             & \texttt{\RuntimeMeanOsetsparse}   & \texttt{\RuntimeMeanOsetdense\phantom{m}}  \\
SID &
\texttt{\char`\~1–2 h\phantom{m}} &
\texttt{\char`\~60 s\phantom{m}} &
\texttt{\char`\~60 s\phantom{m}}
\end{tabular}
\end{center} %

\footnotetext{%
We denote the sparse graphs with $0.75p$ expected edges considered
in \citet{peters2015structural} as extremely sparse (x-sparse);
for x-sparse $1000$-node random graphs,
\citet{peters2015structural} reported a runtime of almost 7000 s,
which on our hardware took \texttt{\char`\~}1 h %
(running $1$ instead of $11$ repetitions).
}

\subsection{Distance Comparison}
\label{sec: distance comparison}

To compare the distances and empirically demonstrate that they capture distinct information,
we draw random pairs of graphs and compute
the Parent-AID, Ancestor-AID, Oset-AID, and SHD between these pairs.
For the distances between $300$ pairs of $30$-node graphs
where $\Gtrue$ is a random dense graph (sampled as described above)
and $\Gguess$ is the graph obtained by removing one edge from $\Gtrue$ at random,
we obtain the following correlation matrix.

\begin{center}\small
\begin{tabular}{l|lll}
\toprule
 & Ancestor-AID & Oset-AID & Parent-AID \\
\midrule
Ancestor-AID & 1 & 0.7281 & 0.0886 \\
Oset-AID & 0.7281 & 1 & 0.2080 \\
Parent-AID & 0.0886 & 0.2080 & 1 \\
\bottomrule
\end{tabular}
\end{center}

Further, the average distances are Ancestor-AID: $2.0$, Oset-AID: $5.9$, and Parent-AID: $11.2$
(while the SHD between all these graph pairs is $1$). We provide a corresponding scatter plot between the distances in \cref{fig:dense-removal}, \cref{app:removal}.
The results highlight that the number of wrongly inferred causal effects if we delete an edge from the true DAG, depends on the choice of identification strategy. 

When benchmarking causal discovery algorithms, the distance should be chosen in line with
the downstream task the graph will eventually be used for.
If the task is to reason about the existence of direct cause-effect relationships, the SHD is a natural choice.
If the task is to infer causal effects, there are multiple options.
The Parent-, Ancestor-, and Oset-AID are three such options, each corresponding to different assumptions on the behavior of an idealized practitioner who will use $\Gguess$ to infer causal effects.
This simulation experiment and additional experiments in \cref{app:comparison} underline that the choice of distance is practically important when benchmarking causal discovery algorithms.

\section{Discussion}

Our framework gives a recipe
for developing distances for other graph types, such as maximal ancestral graphs that allow for hidden variables:
Find a sound and complete identification strategy and a corresponding verifier.
While the adjustment-based identification strategies we use for DAGs and CPDAGs are not sound and complete for settings with hidden variables, sound and complete alternatives exist \citep{huang2006identifiability,shpitser2006identification}.
Yet, there are no verifiers for these alternatives.
Therefore, advances on causal effect identification and in particular verification are needed before we can develop distances for
other graph types so as to aid the development of causal discovery under latent confounding.
Nonetheless, the framework for strategy-specific identification distances provides a handbook on how to develop such a distance as the necessary methodology for causal effect identification and verification becomes available.

\begin{acknowledgements} %
We thank Alexander G.\@ Reisach for valuable discussions and feedback on an earlier draft of the present manuscript.
We also thank the anonymous 
reviewers for constructive feedback that helped improve the presentation.
\end{acknowledgements}

\bibliography{references}

\newpage

\onecolumn

\title{\Title\\(Supplementary Material)}
\maketitle

\appendix

\section{Additional Preliminaries}\label{supp:preliminaries}

\paragraph{Simple graphs with directed and undirected edges.}
A simple graph $\mathcal{G}=(\mathbf{V},\mathbf{E})$ over nodes $\mathbf{V}=\{V_i\mid i\in[d]\}$
with edges $\mathbf{E}$
is a graph where there is at most one edge between any two nodes.
A graph is directed, if all edges are directed edges $\rightarrow$,
and partially directed, if all edges are directed edges $\rightarrow$ or undirected edges $\arrow$.
Two nodes are adjacent, if an edge connects them. In particular a node is adjacent to itself. 

\paragraph{Walks and paths.} A walk $w$ is a sequence of nodes $(T,...,Y)$ where each successive pair of nodes is adjacent. The nodes $T$ and $Y$ are called endpoint nodes on $w$.
A path $p$ is a sequence of distinct nodes $(T,...,Y)$ and is a special case of a walk.
A walk
$w$ is possibly directed from $T$ to $Y$ if no directed edge along the path is directed towards $T$ (a possibly directed walk is sometimes called possibly causal). 
A walk
$w$ is directed from $T$ to $Y$ if all edges along the path are directed and directed towards $Y$ (a directed walk is sometimes called causal); every directed walk is also a possibly directed walk.
We often consider walks that are not possibly directed as they contain at least one edge facing towards $T$,
for ease, we call such a walk $w$ non-causal.
A directed path from $T$ to $Y$ together with $Y\to T$ forms a cycle.

\paragraph{DAGs and PDAGs.}
A directed acyclic graph (DAG) is a simple directed graph, that is, all edges are directed, that has no cycles.
A partially directed acyclic graph (PDAG) is a simple graph that is partially directed and has no cycles.
A DAG is also a PDAG.
A walk $w$ from a set $\mathbf{T}$ to a set $\mathbf{Y}$ is a walk from some node $T \in \mathbf{T}$ to some node $Y\in\mathbf{Y}$, that is, $T$ and $Y$ are the endpoint nodes of the walk.
The walk $w$ is called proper, if it only contains one node in $\mathbf{T}$.
Given two walks $w=(A,\dots,B)$ and $w'=(B,\dots,C)$ we let $w \oplus w' = (A,\dots,B,\dots,C)$ denote the walk we obtain by concatenating $w$ and $w'$.

\paragraph{Node relationships in DAGs and PDAGs.}
If the edge $T \to Y$ or $T \arrow Y$ exists,
$T$ is a possible parent of $Y$ and $Y$ a possible child of $T$.
Let
$\posspa(Y,\g)$
denote the set of all possible parents of $Y$
and 
$\possch(T,\g)$ the set of all possible children of $T$.
If there is a possibly directed path from $T$ to $Y$ or if $T=Y$,
$T$ is a possible ancestor of $Y$
and $Y$ a possible descendant of $T$. Let
$\possan(Y,\g)$
denote the set of all possible ancestors of $Y$
and
$\possde(T,\g)$
the set of all possible descendants of $T$.
If all edges are directed, we analogously define the set of parents $\pa(Y,\g)$, children $\ch(T,\g)$, ancestors $\an(Y,\g)$, and descendants $\de(T,\g)$.
For a set $\mathbf{T}$ we define $\posspa(\mathbf{T},\g) = \bigcup_{T \in \mathbf{T}} \posspa(T,\g)$;
we analogously define $\possch(\mathbf{T},\g),\possan(\mathbf{T},\g),\possde(\mathbf{T},\g),\ch(\mathbf{T},\g),\pa(\mathbf{T},\g),\de(\mathbf{T},\g)$, and $\an(\mathbf{T},\g)$.
We also define\linebreak $\mathrm{NonDe}(T,\g) = \mathbf{V} \setminus \possde(T,\g)$ which in the DAG case reduces to $\mathrm{NonDe}(T,\g) = \mathbf{V} \setminus \de(T,\g)$. %

\paragraph{Supergraph.} A graph $\g=(\mathbf{V},\mathbf{E})$ is a called a supergraph of a graph $\g'=(\mathbf{V},\mathbf{E}')$ if  $\mathbf{E}' \subseteq \mathbf{E}$. We say that $\g$ is a super-DAG of $\g'$ if $\g$ is a DAG, and define super-CPDAG analogously.

\paragraph{Colliders, v-structures, and definite-status paths.}
A node $V$ in a PDAG $\g$ is a collider on a path $p$ if $p$ contains the subpath $U \rightarrow V \leftarrow W$. 
Node $V$ on a path $p$ is called a definite non-collider on $p$ if $p$ contains a subpath such that (a) $U \leftarrow V$, (b) $V \rightarrow W$, or (c) $U \arrow V \arrow W$ and $U$ and $W$ are not adjacent in $\g$. A path $p$ is of definite status if every node on $p$ is a collider, a definite-status non-collider, or an endpoint node.
We define all terms analogously for walks.
A path of the form $U \rightarrow V \leftarrow W$ in $\g$ is called a v-structure if $U$ and $V$ are not adjacent in $\g$.

\paragraph{Blocking and d-separation in PDAGs.} 
Let $\mathbf{Z}$ be a set of nodes in a PDAG $\g$.
A definite-status path $p$ is blocked given $\mathbf{Z}$ if $p$ either contains a non-collider $N \in \mathbf{Z}$ or a collider $C$ such that $\de(C,\g) \cap \mathbf{Z} = \emptyset$. 
A definite-status walk $w$ is blocked given $\mathbf{Z}$ if $p$ either contains a non-collider $N \in \mathbf{Z}$ or a collider $C$ such that $C \notin \mathbf{Z}$.
A definite-status path or walk that is not blocked given $\mathbf{Z}$ is said to be open or d-connecting given $\mathbf{Z}$. 
Given three pairwise disjoint node sets $\mathbf{T},\mathbf{Y},\mathbf{Z}$ in a PDAG $\g$ we say that $\mathbf{T}$ is d-separated from $\mathbf{Y}$ given $\mathbf{Z}$ in $\g$ and write $\mathbf{T} \perp_{\g} \mathbf{Y} \mid \mathbf{Z}$ if $\mathbf{Z}$ blocks all definite-status paths from $\mathbf{T}$ to $\mathbf{Y}$ or equivalently all definite-status walks. If it does not, we say that 
$\mathbf{T}$ and $\mathbf{Y}$ are d-connected given $\mathbf{Z}$. 

\paragraph{CPDAGs, Markov property, and Markov equivalence.} 
A density $f$ is Markov with respect to a DAG $\g$ if for any three pairwise disjoint node sets $\mathbf{T},\mathbf{Y},\mathbf{Z}$ in $\g$ such that $\mathbf{T} \perp_{\g} \mathbf{Y} \mid \mathbf{Z}$, $\mathbf{T}$ is conditionally independent of $ \mathbf{Y}$ given $\mathbf{Z}$. 
Two DAGs that encode the same set of d-separation statements are called Markov equivalent.
Given a DAG $\g$, the set of all DAGs that are Markov equivalent to $\g$ is called a Markov equivalence class and can be represented by a completed partially directed acyclic graph (CPDAG); a special subset of PDAGs characterized by \citet{meek1995causal}. Note that a DAG is in general not a CPDAG.
Given a CPDAG $\mathcal{C}$ let $[\mathcal{C}]$ denote the corresponding equivalence class. For any DAG $\mathcal{D} \in [\mathcal{C}]$ we say that $\mathcal{C}$ is the CPDAG of $\mathcal{D}$. The CPDAG $\mathcal{C}$ has the same v-structures and adjacencies as any DAG in $[\mathcal{C}]$. Further, an edge $\rightarrow$ in $\mathcal{C}$ implies that every DAG in $[\mathcal{C}]$ contains that edge $\rightarrow$.
An edge $\arrow$ in $\mathcal{C}$ implies that some DAG in $[\mathcal{C}]$ contains the edge $\rightarrow$ and some other DAG in $[\mathcal{C}]$ contains the edge $\leftarrow$.

\paragraph{Causal and forbidden Nodes.}
Given set $\mathbf{T}$ and $\mathbf{Y}$ in a DAG or CPDAG $\g$, we define the causal nodes $\cnb{\g}$ to be all nodes that lie on proper directed paths from $\mathbf{T}$ to $\mathbf{Y}$ and are not in $\mathbf{T}$. In a CPDAG $\g$, we define the possibly causal nodes $\posscnb{\g}$ to be all nodes that lie on proper possibly directed paths from $\mathbf{T}$ to $\mathbf{Y}$ and are not in $\mathbf{T}$. We define the forbidden nodes as $\forb(\mathbf{T},\mathbf{Y},\g) = \possde(\posscnb{\g},\g) \cup \mathbf{T}$.

\paragraph{Amenability.} Let $\mathbf{T}$ and $\mathbf{Y}$ be disjoint node sets in a DAG, CPDAG, MAG or PAG $\g$. We say that $(\g,\mathbf{T},\mathbf{Y})$ is amenable if every proper possibly directed path from $\mathbf{T}$ to $\mathbf{Y}$ begins with an edge $\rightarrow$.

The generalized adjustment criterion by \citet{perkovic2018complete}
provides necessary and sufficient graphical conditions for a set to be a valid adjustment set:
\begin{Definition}[Generalized Adjustment Criterion]\label{def:gac}
Let $\mathbf{T}, \mathbf{Y}$, and $\mathbf{Z}$ be pairwise disjoint node sets in a DAG, CPDAG, MAG or PAG $\g$.
Then $\mathbf{Z}$ satisfies the generalized adjustment criterion relative to $(\mathbf{T}, \mathbf{Y})$ in $\g$ if the following three conditions hold:
\begin{description}
\item[(Amenability)] $(\g,\mathbf{T},\mathbf{Y})$ is amenable, and
\item[(Forbidden set)] $\mathbf{Z}\cap\forb(\mathbf{T},\mathbf{Y},\g) = \emptyset$, and
\item[(Blocking)] all proper definite-status non-causal paths from $\mathbf{T}$ to $\mathbf{Y}$ are blocked by $\mathbf{Z}$ in $\g$.
\end{description}
\end{Definition}

\section{Proofs}\label{supp:proofs}

\begin{proof}[Proof of \cref{prop: supergraph distance} (\nameref{prop: supergraph distance})]
    Let $I=\mathcal{I}(\Gguess,{T},{Y})$ be an identifying formula. Since $\mathcal{I}$ is sound and complete, $I=f(\mathbf{y}\mid \doop{\mathbf{t}})$ for any $f$ compatible with $\Gguess$. Since any $f$ compatible with $\Gguess$ is also compatible with $\Gtrue$
    (since all parent sets in $\Gguess$ are supersets of parent sets in $\Gtrue$)
    it follows that
    $I=f(\mathbf{y}\mid \doop{\mathbf{t}})$ for any $f$ compatible with $\Gtrue$.
    Therefore, $\mathcal{V}(\Gtrue,\mathcal{I}(\Gguess, T, Y))$ returns $\correct$.
    Since this is true for any tuple $({T},{Y})$ our claim follows.
\end{proof}

\begin{proof}[Proof of \cref{lemma: non-descendants} (\nameref{lemma: non-descendants})]
    Since $\forb(T,Y,\g) \setminus \de(T,\g) = \emptyset$, $\mathbf{Z}$ satisfies the forbidden set condition of \cref{def:gac}. Let $p$ be a non-causal path from $T$ to $Y$.
    If $p$ begins with an edge into $T$, it contains a node in $\pa(T,\g)$ 
    which is therefore also a node in $\mathbf{Z}$ that is a non-collider on $p$.
    It follows that $p$ is blocked. If $p$ begins with an edge out of $T$, $p$ is either directed or it must contain a collider $C \in \de(T,\g)$. Since $\mathbf{Z} \cap \de(T,\g)=\emptyset$ and $\de(C,\g) \subseteq \de(T,\g)$ it follows that $p$ is blocked given $\mathbf{Z}$. Therefore, $\mathbf{Z}$ satisfies the adjustment criterion.
\end{proof}

\begin{proof}[Proof of \cref{lemma: causal order} (\nameref{lemma: causal order})]
Here, $\pa(V_t,\Gguess)=V_{t-1}$ for $t-1\geq 1$  and therefore
the parent adjustment strategy $\mathcal{I}_P(\Gguess,V_t,V_y)$ returns
\begin{equation*}
    \begin{cases}
    \int f(v_y \mid v_t,v_{t-1}) f(v_{t-1}) \, \mathrm{d}v_{t-1} & \text{if } y \neq t - 1 \geq 1, \\
    f(v_y\mid v_t) \text{\ \ \small("empty adjustment set") } & \text{if } \phantom{y \neq.\!} t-1 = 0,\\
    f(v_y) & \text{if } y = t-1 \geq 1.
\end{cases}
\end{equation*}
We now apply the verifier $\mathcal{V}^{\mathcal{D}}_{\text{adj}}$.
Consider first inputs to the verifier of the form $I=f(v_y)$.
The DAGs $\Gtrue$ and $\Gguess$ have the same causal ordering and therefore any one of the $p-1$ such inputs with $y=t-1 \geq 1$ is correct.
Consider now the remaining inputs to the verifier of the adjustment form.
Since $V_i \in \pa(V_t,\Gtrue)$ for all $i < t-1$,
no valid adjustment set exists for the effect of $V_t$ on such a $V_i$ in $\Gtrue$.
Further, if $t < y$ then the only valid adjustment set is $\{V_{1},\dots,V_{t-1}\}$. Therefore, $\pa(V_t,\Gguess)$ is a valid adjustment set in $\Gtrue$ if and only if $t=1$ or $t=2$ and $y > t$.
Using $\mathcal{I}_{P}$ we therefore obtain exactly $3p-4 = (p-1) + (p-1) + (p-2)$
identification formulas in $\Gguess$ that are also correct in $\Gtrue$,
while the remaining $(p^2-p) - (3p-4) = p^2-4p+4$ are false.
It follows that 
$$\lim_{p \to \infty} d^{\mathcal{I}_{P}}(\Gguess,\Gtrue) /(p^2-p) = 1.$$
\end{proof}

\begin{proof}[Proof of \cref{prop: ancestral adjustment} (\nameref{prop: ancestral adjustment})]
    Consider two nodes $T$ and $Y$ in a DAG $\Gguess$. If $Y \in \mathbf{D}_T=\de(T,\Gguess)$, then $\mathbf{A}_T=\an(T,\Gguess)$ is a valid adjustment set in $\Gguess$ by \cref{lemma: non-descendants}. If $Y \notin \mathbf{D}_T$, then $f(y\mid \doop{t})=f(y)$. Therefore, $\mathcal{I}_{A}$ is sound and complete. 

    Consider two DAGs $\Gtrue$ and $\Gguess$, such that $\Gguess$ respects the causal orders of $\Gtrue$, that is, $\de(T,\Gtrue) \subseteq \de(T,\Gguess)$ or equivalently $\an(T,\Gtrue) \subseteq \an(T,\Gguess)$ for all nodes $T$. Fix a pair $(T,Y)$ and consider $I_{TY}=\mathcal{I}_{A}(\g,T,Y)$. If $Y \notin \de(T,\Gguess)$ then $Y \notin \de(T,\Gtrue)$ and therefore, $I_{TY}=f(y)$ is a correct identifying formula in $\Gtrue$. If $Y\in\de(T,\Gguess)$ then $Y \notin \an(T,\Gguess)$, which implies that $Y \notin \an(T,\Gtrue)$  and $Y \notin \pa(T,\Gtrue)$. Since $\pa(T,\Gtrue) \subseteq \an(T,\Gtrue) \subseteq \an(T,\Gguess) \subseteq \mathrm{NonDe}(T,\Gguess)$ and $\mathrm{NonDe}(T,\Gguess) = \mathbf{V} \setminus \de(T,\Gguess) \subseteq \mathbf{V} \setminus \de(T,\Gtrue)$ we can therefore invoke \cref{lemma: non-descendants} to conclude that $I_{TY}$ is correct in $\Gtrue$. Therefore, $d^{\mathcal{I}_{A}}(\Gtrue,\Gguess) = 0$.
    
    Suppose now that $\Gguess$ does not respect 
    the causal order of $\Gtrue$, that is, there exists a pair $T,Y$ such that $Y \in \de(T,\Gtrue) \setminus \de(T,\Gguess)$. Since $Y \notin \de(T,\Gguess)$, $I_{TY}=f(y)$ but this is wrong in $\Gtrue$.
\end{proof}

\begin{proof}[Proof of \cref{prop: optimal adjustment DAGs} (\nameref{prop: optimal adjustment DAGs})]
    The soundness and completeness of $\mathcal{I}_O$ follows by the results of \citet{henckel2019graphical} and that the causal effect on a non-descendant is the observational density.
\end{proof}

\section{Additional Examples and Results}
\label{section: additional examples}

\begin{Example}[Counter-Examples where Oset Adjustment Distance is (not) Zero]
    \label{example: zero optimal distance}
    Let  $\Gtrue^1$ be the graph from \cref{subfig: Gtrue zero optimal distance_one}, $\Gtrue^2$ be the graph from \cref{subfig: Gtrue zero optimal distance_two}, and $\Gguess$ be the graph from \cref{subfig: Gguess zero optimal distance}. 
    
    Consider the pair $(\Gtrue^1,\Gguess)$. Since, $V_2$ is an isolated node in $\Gtrue^1$ any identifying formula produced by $\mathcal{I}_O$ for a pair involving $V_2$ will be correct irrespective of $\Gguess$. Further, $\mathcal{I}_O(\Gguess,V_3,V_1)=f(v_1)$ and $\mathcal{I}_O(\Gguess,V_1,V_3)=f(v_3\mid v_1)$. Since $V_1 \notin \de(V_3,\Gtrue^1)$ and the empty set is a valid adjustment set relative to $(V_1,V_3)$ in $\Gtrue^1$ it follows that $d^{\mathcal{I}_O}(\Gtrue^1,\Gguess)=0$. This shows that the Oset adjustment distance may be zero even if $\Gguess$  is not a supergraph of $\Gtrue$.

    Consider the pair $(\Gtrue^2,\Gguess)$. Since $I_O(\Gguess,V_2,V_3)=f(v_3\mid v_2)$ and the empty set is not a valid adjustment set relative to $(V_2,V_3)$ in $\Gtrue^2$, $d^{\mathcal{I}_O}(\Gtrue^2,\Gguess)\neq 0$. This shows that $\Gguess$ respecting the causal order of $\Gtrue$ does not ensure that the Oset adjustment distance is zero.
\end{Example}

\begin{Example}[Correct CPDAG may be further from true DAG than Incorrect CPDAG]
\label{example: DAG to CPDAG}
Let $\Gtrue^D$ be a fully connected DAG. Let $\Gguess^1$ be the corresponding CPDAG, that is, the fully connected CPDAG. Since every effect in $\Gguess$ is non-identifiable it follows that for any strategy-specific distance $d^{\mathcal{I}}(\Gtrue,\Gguess^1)=p(p-1)$. Let $\Gguess^2$ be the empty CPDAG. Since exactly half the effects in $\Gtrue$ are zero it follows that for any identification strategy that uses a descendant check, such as  $\mathcal{I}_A$ or $\mathcal{I}_O$, $d^{\mathcal{I}}(\Gtrue,\Gguess^2)=p(p-1)/2$
    
\end{Example}

\begin{Proposition}[Ancestor Adjustment Strategy is Sound and Complete for CPDAGs]
    \label{proposition: ancestral CPDAGs}
    Consider nodes $T$ and $Y$ in a CPDAG $\g$, such that $(\g,T,Y)$ is amenable. Then $\an(T,\g)$ is a valid adjustment set relative to $(T,Y)$ in $\g$.
\end{Proposition}

\begin{proof}
    Since by assumption $(\g,T,Y)$ is amenable and $\an(T,\g)$ satisfies the forbidden set condition of \cref{def:gac}, it only remains to show that $\an(T,\g)$ satisfies the blocking condition. To see this consider a definite-status non-causal path $p$ from $T$ to $Y$. If $p$ begins with an edge $\leftarrow$ it contains a node in $\pa(T,\g)$ and is therefore blocked given $\an(T,\g)$. If it begins with an edge $\arrow$ or $\rightarrow$, then it must contain at least one collider $C \in \possde(T,\g)$ by the fact that it is of definite status and may therefore not contain $- V \leftarrow$ but contains at least one backwards facing edge. Since $\de(C,\g) \cap \an(T,\g) = \emptyset$, it follows that $p$ is blocked given $\an(T,\g)$.
\end{proof}

\begin{Lemma}[Oset-Characterization Simplifies Given Amenability]
    Consider two node sets $\mathbf{T}$ and $\mathbf{Y}$ in a CPDAG $\g$ such that $(\g,\mathbf{T},\mathbf{Y})$ is amenable. Then 
    $$
    \optb{\g} = \pa(\de(\mathbf{T},\g) \cap \mathrm{PropAn}(\mathbf{Y},\mathbf{T},\g),\g) \setminus \de(\mathbf{T},\g),
    $$
    where $\mathrm{PropAn}(\mathbf{Y},\mathbf{T},\g)$ is the set of all nodes $N$, such that there exists a directed path from $N$ to some $Y \in \mathbf{Y}$ that does not contain any nodes in $\mathbf{T}$.
    \label{lemma: alternative optimal}
\end{Lemma}

\begin{proof}
    Recall that $\optb{\g} = \pa(\mathrm{Cn}(\mathbf{T},\mathbf{Y},\g),\g) \setminus \forb(\mathbf{T},\mathbf{Y},\g)$. Clearly, 
    $$
    \mathrm{Cn}(\mathbf{T},\mathbf{Y},\g) = \de(\mathbf{T},\g) \cap \mathrm{PropAn}(\mathbf{Y},\mathbf{T},\g)
    $$
    by the definition of $\mathrm{Cn}(\mathbf{T},\mathbf{Y},\g)$. Since every causal node is in $\mathrm{PropAn}(\mathbf{Y},\mathbf{T},\g)$ so is every node in $\optb{\g}$. Since every node that is both in $\mathrm{PropAn}(\mathbf{Y},\mathbf{T},\g)$ and $\de(\mathbf{T},\g)$ is a causal node and therefore forbidden it follows that $\optb{\g} \cap \de(\mathbf{T},\g) = \emptyset$. By Lemma E.6 by \citet{henckel2019graphical}, $\forb(\mathbf{T},\mathbf{Y},\g) \subseteq \de(\mathbf{T},\g)$ and therefore removing all nodes in $\de(\mathbf{T},\g)$ from $\pa(\mathrm{Cn}(\mathbf{T},\mathbf{Y},\g),\g)$ is equivalent to removing all nodes in $\forb(\mathbf{T},\mathbf{Y},\g)$. 
\end{proof}

\begin{figure}
\begin{subfigure}{0.3\linewidth}
 \centering
     \begin{tikzpicture}[scale=.9]
      \node[state] (v1) at (0,0) {$V_1$};
      \node[state] (v2) at (1.5,0) {$V_2$};
      \node[state] (v3) at (3,0) {$V_3$};
      \path (v1) edge[directed, bend left] (v3);
    \end{tikzpicture}
    \caption{– $\Gtrue^1$}
    \label{subfig: Gtrue zero optimal distance_one}
\end{subfigure}
\hfill
\begin{subfigure}{0.3\linewidth}
 \centering
\begin{tikzpicture}[scale=.9]
      \node[state] (v1) at (0,0) {$V_1$};
      \node[state] (v2) at (1.5,0) {$V_2$};
      \node[state] (v3) at (3,0) {$V_3$};
      \path (v1) edge[directed] (v2)
            (v2) edge[directed] (v3)
            (v1) edge[directed] (v2);
    \end{tikzpicture}
    \caption{– $\Gguess$}
    \label{subfig: Gguess zero optimal distance}
\end{subfigure}
\hfill
\begin{subfigure}{0.3\linewidth}
 \centering
     \begin{tikzpicture}[scale=.9]
      \node[state] (v1) at (0,0) {$V_1$};
      \node[state] (v2) at (1.5,0) {$V_2$};
      \node[state] (v3) at (3,0) {$V_3$};
      \path (v1) edge[directed, bend left] (v3)
            (v1) edge[directed] (v2)
            (v2) edge[directed] (v3);
    \end{tikzpicture}
    \caption{– $\Gtrue^2$}
    \label{subfig: Gtrue zero optimal distance_two}
\end{subfigure}
    \caption{DAGs for \cref{example: zero optimal distance}.}
    \label{fig: zero optimal distance}
\end{figure}

\section{Algorithm Development}
\label{section: appendix implementation}

Many graph properties involved in the validation of adjustment sets,
such as amenability, forbidden nodes, or blocking,
are based on the (non-)existence of paths with certain properties.
For example, two nodes $T$ and $Y$ are d-connected in a DAG given $\mathbf{Z}$
if and only if a path exists between them that is not blocked by $\mathbf{Z}$.
We can reformulate the problem of verifying whether such a path exists as a reachability task:
Starting from $T$, try reaching $Y$ by following all possible paths that are not blocked by $\mathbf{Z}$
until either you reach $Y$ or you have exhausted all possible paths.

The \emph{Bayes-Ball} algorithm \citep{geiger1989dseparation,shachter1998bayes} %
uses the reachability framework to obtain all nodes in a DAG that are d-connected to $T$ given $\mathbf{Z}$.
A key insight for its efficient implementation is that a d-connecting walk exists between two nodes if and only if a d-connecting path exists between them;
thus, we can avoid having to
a) check for each collider along the path that one of its descendants is in $\mathbf{Z}$ before continuing
and to
b) follow all possible paths for which we need to keep track of all previously visited nodes along any given path traversal
(to avoid visiting the same node more than once).
Instead, we can traverse along walks and continue a walk from $V$ to $W$ along an incoming/outgoing edge if
a') $W$ has not previously been reached through an incoming/outgoing edge
and if
b') the walk is not blocked by $\mathbf{Z}$
(if we face $\to V \gets W$, we continue from $V$ to $W$ only if $V \in \mathbf{Z}$,
otherwise we continue only if $V \not\in\mathbf{Z}$).
The benefit of a') and b') over a) and b) is that
both conditions are local to the current node in a walk and verifiable
without querying ancestor sets or storing and checking against all previously visited nodes.
In this walk reachability algorithm,
each node is visited at most twice and each edge considered a constant number of times \citep[see, for example, Appendix A in][]{wienobst2022finding}
and therefore its runtime is linear in the number of nodes $p$ plus the number of edges $m$;
in dense graphs the number of edges grows quadratic in the number of nodes and consequently the runtime is $O(p^2)$ for dense graphs.

\citet{wienobst2022finding} generalize the algorithmic concept underlying \emph{Bayes-Ball} to a class of DAG search algorithms
akin to a depth-first graph search that recursively visits neighbouring nodes that have not been reached by the same kind of edge before.
Since each node $V$ is visited at most once per edge type (for example, $\to V$ or $\gets V$ in DAGs),
the runtime of their \emph{gensearch} algorithm is also $O(p^2)$.
Rule tables encode the conditions for continuing on a given walk
based on how the current node was reached and how the potential next node $W$ would be reached.
Which rule table we use for the \emph{gensearch} algorithm
determines the properties of the nodes that will be reached and therefore returned by the algorithm; for example,
they show that a sequence of \emph{gensearch} algorithms with carefully chosen rule tables
finds a minimal adjustment set in $O(p+m)$ time.

Key to implementing our adjustment distances efficiently,
is a new walk-status-aware reachability algorithm
that, given a DAG or CPDAG $\g$ with $p$ nodes and $m$ edges,
set $\mathbf{Z}$, and treatment nodes $\mathbf{T}$,
returns all nodes $Y$ such that $\mathbf{Z}$
is a valid adjustment set for $(\mathbf{T},Y)$ in $\g$
in $O(p+m)$ time.
We use this algorithm to verify an adjustment set for many $Y$ simultaneously.
To implement this algorithm, 
we prove a modified adjustment criterion,
adapt the reachability algorithm for finding d-connected nodes in DAGs
to account for walks that are not of definite-status in CPDAGs,
implement an amenability check,
and show how to find nodes that satisfy all conditions of the modified adjustment criterion with only one reachability algorithm.
We proceed as follows:
\begin{itemize}

\item In \cref{app:modifiedcriterion}
we prove a modified adjustment criterion that translates all conditions of the generalized adjustment criterion \citep{perkovic2018complete} into conditions on the (non-)existence of certain walks.
An adjustment criterion in terms of walks allows us to verify whether it holds using only reachability algorithms.

\item In \cref{app:blockingincpdags}
we show how to verify blocking in
CPDAGs with a reachability algorithm that uses no non-local information to verify whether a walk is definite-status or not.

\item In \cref{app:walkstatus} we provide motivation and intuition for our decision to add a walk-status to reachability algorithms that is propagated forward in the depth-first search traversal of the graph.
The addition of a walk status allows us to track walks that do not transmit reachability,
but may change status and become walks for which we need to track nodes reached by such a walk.

\item
In \cref{app:efficientalg}, we demonstrate how we use our new reachability algorithm to calculate the Parent- and Ancestor-AID between DAGs or CPDAGs with $O(p(p+m))$ and the Oset-AID with $O(p^2(p+m))$ time complexity.

\end{itemize}

\subsection{Modified Adjustment Criterion}
\label{app:modifiedcriterion}

We prove a modified adjustment criterion that translates all conditions of the generalized adjustment criterion \citep{perkovic2018complete} into conditions on the (non-)existence of certain walks.
Having an adjustment criterion exclusively in terms of walks allows us to use reachability algorithms to verify it.

\begin{Lemma}[Modified Adjustment Criterion for Walk-Based Verification]
    Consider nodes $\mathbf{T}$ and $\mathbf{Y}$ in a DAG or CPDAG $\g$ and a node set $\mathbf{Z}$ in $\g$. The set $\mathbf{Z}$ fulfills the adjustment criterion if and only if
    \begin{enumerate}
        \item every proper possibly directed walk from $\mathbf{T}$ to $\mathbf{Y}$ begins with a directed edge out of $\mathbf{T}$,  and
        \item no proper possibly directed walk from $\mathbf{T}$ to $\mathbf{Y}$ contains a node in $\mathbf{Z}$, and
        \item every proper definite-status walk from $\mathbf{T}$ to $\mathbf{Y}$ that contains a backwards facing edge is blocked by $\mathbf{Z}$.
    \end{enumerate}
    \label{lemma: alternative adjustment criterion}
\end{Lemma}

\begin{proof}
We prove our claims for the CPDAG case as the DAG case can be shown with the same basic arguments but is simpler. We first show that if $\mathbf{Z}$ does not satisfy the adjustment criterion then it also does not satisfy the alternative adjustment criterion. Since the two criteria both assume amenability we can assume amenability holds. Suppose that there exists a proper definite-status non-causal path $p$ from $\mathbf{T}$ to $\mathbf{Y}$ that is open given $\mathbf{Z}$. Consider all colliders $C_1, \dots, C_k$ and the corresponding directed paths $q_1,\dots,q_k$ from $C_i$ to some node  $Z_i \in \mathbf{Z}$. If any of the $q_i$ contains a node in $T' \in \mathbf{T}$ we can replace $p$ with $q_i(T',C_i) \oplus p(C_i,Y)$ so without loss of generality we can assume this is not the case. By appending the $q_i$ to $p$ we obtain a proper definite-status walk from $\mathbf{T}$ to $\mathbf{Y}$ that contains an edge $\leftarrow$, inherited from $p$. 

We can therefore assume, no such $p$ exists and that $\mathbf{Z} \cap \forb(\mathbf{T},\mathbf{Y},\g) \neq \emptyset$. Since
$\forb(\mathbf{T},\mathbf{Y},\g) =$\linebreak $\possDe(\posscnb{\g},\g)$ %
we can in fact assume that there exists a node $Z \in \mathbf{Z} \cap (\possDe(\posscnb{\g},\g) \setminus \posscnb{\g})$. For such a node $Z$ there exists a directed path $p_1$ from $T$ to $Z$ by Lemma E.6 of \citet{henckel2019graphical} and a possibly directed path $p_2$ from some causal node $N$. Since $N$ is possibly causal there also exists a possibly directed path $p_3$ from $N$ to some node  in $Y \in \mathbf{Y}$. We can choose all three paths to not contain other nodes in $\mathbf{Z}$. Let $I$ be the node closest to $Z$, where $p_2$ and $p_3$ intersect and consider $p_4=p_2(Z,I) \oplus p_3(I,Y)$. Note that $p_4$ is colliderless and by taking shortcuts we obtain a definite-status path $p^*_4$, that is also colliderless. By assumption on $Z$, $p^*_4$ must be non-causal and therefore the walk $w=p_1 \oplus p_4^*$ has $Z$ as a definite-status collider, contains no other node in $\mathbf{Z}$ and all other nodes are endpoint nodes or definite-status non-colliders. The walk $w$ therefore violates the blocking condition of the modified criterion. 

We now show that if $\mathbf{Z}$ satisfies the adjustment criterion then it satisfies the alternative adjustment criterion. Again we can assume amenability holds. It suffices to show that any proper non-causal definite-status walk $w$ from $T \in \mathbf{T}$ to $Y\in \mathbf{Y}$ is blocked given $\mathbf{Z}$. Suppose $w$ is colliderless. This in particularly implies that $w$ begins with an edge $T \leftarrow N$. By snipping cycles and appropriate shortcuts we obtain a definite-status, colliderless, non-causal path, where we use that $N$ will always be of definite status and therefore the directed edge into $T$ will never be removed. This path $p$ is blocked given $\mathbf{Z}$ and therefore contains a node in $\mathbf{Z}$. Therefore, so does $w$ which implies that it is blocked given $\mathbf{Z}$. Suppose now that $w$ contains colliders $C_1,\dots,C_k$. By snipping cycles and taking shortcuts we can again obtain a definite-status path $p$ from $T$ to $Y$. Suppose $p$ is possibly directed, that is, consists of possibly causal nodes. Then at least on of the colliders must be a descendant of a possibly causal node and therefore $w$ is blocked. If $p$ is not possibly directed it must either contain a non-collider in $\mathbf{Z}$ that is also a non-collider on $w$ or a collider $C$, such that $\de(C,\g) \cap \mathbf{Z} \neq \emptyset$. In the former case, $w$ is obviously blocked. In the latter case, at least one of the $C_i$ must satisfy $C_i \notin \mathbf{Z}$ which again suffices to conclude that $w$ is blocked.
\end{proof}

With this modified adjustment criterion we can
algorithmically verify that a set $\mathbf{Z}$ is a valid adjustment set for $(\mathbf{T}, Y)$ in DAG or CPDAG $\g$,
by verifying that
\begin{enumerate}

\item no proper possibly directed walk that does not begin with a directed edge out of $\mathbf{T}$
reaches $Y$, and

\item no proper possibly directed walk that contains a node in $\mathbf{Z}$
reaches $Y$, and

\item no proper definite-status non-causal walk that is not blocked by $\mathbf{Z}$
reaches $Y$.

\end{enumerate}
Condition 1.\@ is equivalent to Condition 1.\@ in \cref{lemma: alternative adjustment criterion};
since we need to verify that no such walk reaches $Y$, the reachability algorithm will need to continue walking possibly directed walks that do not start with an edge out of $\mathbf{T}$ even if they are blocked by $\mathbf{Z}$.
Condition 2.\@ is equivalent to Condition 2.\@ in \cref{lemma: alternative adjustment criterion};
since we need to verify that no such walk reaches $Y$, the reachability algorithm will need to continue walking possibly directed walks that start with an edge out of $\mathbf{T}$ even if they contain a node in $\mathbf{Z}$.
Condition 3.\@ is equivalent to the blocking Condition 3.\@ in \cref{lemma: alternative adjustment criterion};
since the (non-)existence of blocked non-causal paths does not appear in any of the conditions, the reachability algorithm can stop walking non-causal walks upon reaching a blocking node in $\mathbf{Z}$;
however, the blocking condition poses a problem as verifying whether a path is blocked requires a non-local check to verify that it is of definite-status.
In the following subsection, we show how to verify blocking in CPDAGs with a reachability algorithm while avoiding this non-local definite-status check.

\subsection{D-separation via a Reachability Algorithm for CPDAGs}
\label{app:blockingincpdags}

In this section, we show how to verify blocking in CPDAGs by a reachability algorithm without needing to discern the
non-local property of a walk being definite-status or not.
This is necessary to enable the use of reachability algorithms with local decision rules to verify the blocking condition on definite-status walks in the modified adjustment criterion.
We show this in 5 steps:
\begin{itemize}
    \item \cref{lemma: definite status suffices}: \nameref{lemma: definite status suffices}
    \item \cref{lemma: path iff walk}: \nameref{lemma: path iff walk}
    \item \cref{lemma: naive reachable works}: \nameref{lemma: naive reachable works}
    \item \cref{lemma: smarter reachable works}: \nameref{lemma: smarter reachable works}
    \item \cref{lemma: complicated reachable works too}: \nameref{lemma: complicated reachable works too}
\end{itemize}

The following lemma by \citet{henckel2019graphical} characterizes d-separation in a CPDAG 
in terms of definite-status paths.

\begin{Lemma}[Indefinite-status paths are irrelevant for d-separation in CPDAGs]
Consider node sets $\mathbf{T},\mathbf{Y}$ and $\mathbf{Z}$ in a CPDAG $\g$. Then $\mathbf{T}$ is d-separated from $\mathbf{Y}$ given $\mathbf{Z}$ in every DAG $\mathcal{D} \in [\g]$ if an only if every definite-status path from $\mathbf{T}$ to $\mathbf{Y}$ is blocked given $\mathbf{Z}$ in $\g$.
\label{lemma: definite status suffices}
\end{Lemma}

Checking whether a collider is open on a definite-status path requires checking a non-local condition, as we need to consider all descendants of the collider. In DAGs we can circumvent that by considering walks instead, as it is possible to show that an open path exists if and only if an open walk exists. We now show that a similar result holds for CPDAGs, connecting definite-status paths and connecting definite-status walks. 

\begin{Lemma}[Existence of open definite-status walks or paths coincides in CPDAGs]
Consider node sets $\mathbf{T},\mathbf{Y}$ and $\mathbf{Z}$ in a CPDAG $\g$. Then there exists a definite-status path from $\mathbf{T}$ to $\mathbf{Y}$ that is open given $\mathbf{Z}$ if and only if there exists a definite-status walk from $\mathbf{T}$ to $\mathbf{Y}$ that is open given $\mathbf{Z}$.
\label{lemma: path iff walk}
\end{Lemma}

\begin{proof}
    Let $p$ be a definite-status path from some $T\in \mathbf{T}$ to some $Y\in \mathbf{Y}$ that is open given $\mathbf{Z}$ in $\g$. Let $C_1,\dots,C_k$ be all colliders on $p$. By assumption there exist directed paths $q_1,\dots,q_k$ from $C_i$ to some node $Z_i \in \mathbf{Z}$ that we choose to not contain any other node in $\mathbf{Z}$. Then $w = p(T,C_1) \oplus q_1(C_1,Z_1) \oplus q_1(Z_1,C_1) \oplus \dots \oplus p(C_k,Y)$ is a definite-status walk from $T$ to $Y$ that is open given $\mathbf{Z}$.

    For the converse direction consider a walk $w$ and let $I$ be the node closest to $T$ on $w$ that appears twice $w$, i.e., $w=w(T,I) \oplus w(I,I) \oplus w(I,Y)$. Consider the walk $w'=w(T,I) \oplus w(I,Y)$. We will now show that either $w'$ itself is a definite-status walk from $T$ to $Y$ such that no no-collider is in $\mathbf{Z}$ and every collider has a descendant in $\mathbf{Z}$ or that we can construct a shortcut walk that is. Since $w'$ contains at least one repeating node less than $w$ we can then iterate this contraction to obtain a definite-status path open given $\mathbf{Z}$. Every node on $w'$ inherits their definite status from the path $w$ except for $I$. Suppose $I \in \mathbf{Z}$, then $I$ must be a collider whenever it appears on $w$ and therefore it also a definite-status collider on $w'$ which is therefore of the claimed form. Suppose now that $I \notin \mathbf{Z}$. Then it must be a non-collider whenever it appears on $w$. There are three cases to consider: a) $I$ is a collider on, $w'$ b) $I$ is a definite-status non-collider on $w'$ and c) $I$ is not of definite status on $w'$. In case a) $w$ must have been of the form $T \cdots \rightarrow I \rightarrow \cdots \leftarrow I \leftarrow \cdots Y$. Therefore $w$ must contain a collider that is a descendant of $I$. Therefore $\de(I,\g) \cap \mathbf{Z} \neq \emptyset$. In case b) $w'$ trivially fulfills the required conditions. 
    
    In case c) we again consider three subcases: a) $A \rightarrow I - B$, b) $A - I \leftarrow B $ and c) $A - I - B$. In all three cases $A$ and $B$ are definite-status non-colliders on $w'$ as they inherited their status from $w$ and they cannot have been colliders. This also implies that $A,B \notin \mathbf{Z}$. In case a) there must also exist an edge $A \rightarrow B$ and we can replace $w'$ with $w'(T,A) \oplus w'(B,Y)$. The node $A$ is a definite non-collider on this new walk. If $B$ is also we are done. If it is not, i.e., we have the structure $A \rightarrow B - B'$ we can replace $w'$ again by repeating the argument we just made and taking the shortcut to $B'$. We can do so iteratively, until we encounter either a definite-status non-collider or $Y$ itself. Either way we obtain a definite-status walk such that every non-collider is not in $\mathbf{Z}$ and every collider has a descendant that is. Case b) follows by the exact same argument reversing the roles of $A$ and $B$. In case c) we must have an edge $A - B$. Again we replace $w'$ by taking the shortcut. If $A$ is not of definite status on the new walk, i.e, it contains the segment $A' - A - B$, there must exist an edge $A' - C$. We can again iteratively take shortcuts until we either obtain an $A'$ that is a definite-status non-collider or arrive at $T$. If $B$ is not of definite status we repeat the same procedure untile we arrive at definite-status $B'$ or $Y$. In all cases, we obtain a walk $w'$ of definite status that is open given $\mathbf{Z}$.
\end{proof}

Based on \cref{lemma: path iff walk} we can propose a reachability d-separation algorithm that additionally tracks and discerns whether a walk in a CPDAG is definite-status or not. The algorithm is a gensearch algorithm \citep[][Algorithm 6]{wienobst2022finding} using the rule table given in \cref{table:reachable naive} to traverse the graph. We now prove that this table is correct for d-separation in CPDAGs.

\begin{Lemma}[Reachability algorithm with non-local decision rules for d-connectedness in CPDAGs.]
    Consider a node set $\mathbf{T}$ in a CPDAG $\g$ and let $\mathbf{Z}$ be a node set in $\g$.
    The output of a reachability algorithm (gensearch by \citet{wienobst2022finding}) with the rule table given in \cref{table:reachable naive} is the set of all nodes $Y \in \mathbf{V}$ that are d-connected with $\mathbf{T}$ given $\mathbf{Z}$ in $\g$.
    \begin{table}[h]
    \centering
    \begin{tabular}{c|c|c}
         case& continue to $W$ & yield $W$ \\ \midrule
         init $T - W$ & always &  always \\
         init $T \rightarrow W$ & always & always\\
         init $T \leftarrow W$ & $W\notin \mathbf{Z}$ & always\\
         $ - V - W$ & $V \notin \mathbf{Z}$ and $V$ of definite status &  $V \notin \mathbf{Z}$ and $V$ of definite status\\
         $ - V \rightarrow W$ & $V \notin \mathbf{Z}$ & $V \notin \mathbf{Z}$\\
         $ - V \leftarrow W$ & never &  never \\
         $ \rightarrow V - W$ & never & never \\
         $ \rightarrow V \rightarrow W$ & $V \notin \mathbf{Z}$  & $V \notin \mathbf{Z}$ \\
         $ \rightarrow V \leftarrow W$ & $V \in \mathbf{Z}$ & $V \in \mathbf{Z}$ \\
         $ \leftarrow V - W$ & $V \notin \mathbf{Z}$ & $V \notin \mathbf{Z}$ \\
         $ \leftarrow V \rightarrow W$ & $V \notin \mathbf{Z}$ & $V \notin \mathbf{Z}$ \\
         $ \leftarrow V \leftarrow W$ & $V \notin \mathbf{Z}$ & $V \notin \mathbf{Z}$ \\
    \end{tabular}
    \caption{Rule table for gensearch algorithm \citep{wienobst2022finding} to compute all d-connected nodes in a CPDAG.}
    \label{table:reachable naive}
\end{table}

\label{lemma: naive reachable works}
\end{Lemma}

\begin{proof}
    Every node adjacent to some $T\in \mathbf{T}$ is d-connected with $\mathbf{T}$ given $\mathbf{Z}$. Therefore, the initialization step of the reachability algorithm is correct. Now suppose that if we arrive at a node $V$ in the reachability algorithm that there exists some definite-status walk $w$ from some $T \in \mathbf{T}$ to $V$ that is open $\mathbf{Z}$ and consider a proper step of the reachability algorithm continuing from $V$. Based on the rule table we continue and yield $W$ precisely when appending the edge between $V$ and $W$ to $w$ results in a definite-status d-connecting walk $w'$ from $T$ to $W$, i.e., if $V$ is definite non-collider on $w'$ and $V \notin \mathbf{Z}$ or if $V$ is a collider on $w'$ and $V \in \mathbf{Z}$. By induction it follows that for every reachable node $Y$ there exists a definite-status d-connecting walk from $T$ which by \cref{lemma: path iff walk} suffices to conclude that $\mathbf{T}$ and $Y$ are d-connected given $\mathbf{Z}$.

    We now show that if a node $Y$ is d-connected with some $T \in \mathbf{T}$ given $\mathbf{Z}$ then it will be returned as reachable. By \cref{lemma: path iff walk} there exist a definite-status d-connecting walk from $T$ to $Y$. We now make an induction argument on the length $l$ of the shortest such walk. If $l=1$, the algorithm clearly returns $Y$, so suppose the algorithm returns all nodes with shortest paths of length $l=k-1$ and suppose for $Y$ the shortest walk $w$ is of length $l=k$. Let $w'$ be the walk we obtain by removing the final edge from $w$. It's a walk of length $l$ that is d-connecting given $\mathbf{Z}$ and therefore this holds for it's end node $Y'$. By the induction hypothesis $Y'$ is reachable. Further, since $w$ is a definite-status d-connecting walk we can see that by applying the rule table to $Y'$ the algorithm will also return $Y$.
\end{proof}

From an implementation perspective, it is problematic that in the fourth row of \cref{table:reachable naive} we need to check whether $V$ is of definite status, since this requires storing adjacent nodes for previously visited nodes, that is, this rule is non-local. We now show that we can simply drop this check without modifying the output of the algorithm and in this way obtain a local algorithm.

\begin{Lemma}[We may treat some indefinite-status walks as definite-status]
    Consider a node set $\mathbf{T}$ in a CPDAG $\g$ and let $\mathbf{Z}$ be a node set in $\g$. Suppose we modify the rule table in \cref{table:reachable naive} by treating the $\arrow V\arrow W$ case as if $V$ were of definite status (irrespective of whether it actually is).
    The resulting reachability algorithm has the same output as the original algorithm.
    \label{lemma: smarter reachable works}
\end{Lemma}

\begin{proof}
    The two algorithms agree locally in all cases except in the case $-V-W$, with $V \notin \mathbf{Z}$ not of definite status, where the original algorithm does not continue to $W$ and the modified algorithm does and also considers all vertices such that $W - W'$ or $W \rightarrow W'$ if $W \notin \mathbf{Z}$. Suppose that starting from some $T \in \mathbf{T}$, $V$ is the first node where the two algorithms disagree, i.e., there exists a $V'$ reachable by both algorithms such that $V' - V - W$, $V \notin \mathbf{Z}$ and $V' - W$. Since $V'$ is reachable with the original algorithm and we continue onto an undirected edge, $V' \notin \mathbf{Z}$ and we arrive at $V'$ via an edge of the form $\leftarrow V'$ or $-V$. The original algorithm will therefore reach $W$, either via $\leftarrow V'-W$ or $-V'-W$ unless in the latter case $V'$ is not of definite status. If the latter is the case we can repeat the argument to obtain a new $V'$ until we either arrive at a $V'$ that is of definite status or the walk $T - W$. In either case, $W$ is reachable and if $W \notin \mathbf{Z}$ we will consider all vertices such that $W - W'$ or $W \rightarrow W'$.
\end{proof}

Finally, we will use a d-separation reachability algorithm within our new reachability algorithm to verify adjustment validity. Here, we also need to verify whether $\mathbf{Z}$ contains possibly causal nodes, that is, whether there exists a proper possibly directed path from $\mathbf{T}$ to $\mathbf{Y}$ that contains a node in $\mathbf{Z}$ (see \cref{lemma: alternative adjustment criterion}). To do so, we may have to continue along segments of the form $\rightarrow N \arrow$ if $N \notin \mathbf{Z}$ which are of indefinite status.
We now show that we can further modify the d-separation rule tables to accommodate this without changing the output of the d-separation reachability algorithm.
This will allow us to run all three checks required in the validity algorithm simultaneously.
If we were only interested in d-separation, the rule table from \cref{lemma: smarter reachable works} is more computationally efficient.

\begin{Lemma}[We may treat some more indefinite-status walks as definite-status]
    Consider a node set $\mathbf{T}$ in a CPDAG $\g$ and let $\mathbf{Z}$ be a node set in $\g$. Suppose we modify the rule table in \cref{table:reachable naive} by treating the $\arrow V\arrow W$ case as if $V$ were of definite status and the $\rightarrow V \arrow$ case by proceeding if $V \notin \mathbf{Z}$ (analogous to the rule for definite-status non-colliders).
    The resulting reachability algorithm has the same output as the original algorithm.
    \label{lemma: complicated reachable works too}
\end{Lemma}

\begin{proof}
We have already established in \cref{lemma: smarter reachable works}, that we can ignore the definite-status check without modifying the output so consider an algorithm based on this rule table and compare it to an algorithm with the additional rule modification stated in the lemma. The two algorithms agree locally in all cases except in the case $\rightarrow V-W$, with $V \notin \mathbf{Z}$, where the original algorithm does not continue to $W$ and the modified algorithm does and also considers all vertices such that $W - W'$ or $W \rightarrow W'$ if $W \notin \mathbf{Z}$. Suppose that starting from some $T \in \mathbf{T}$, $V$ is the first node where the two algorithms disagree, i.e., there exists a $V'$ reachable by both algorithms such that $V' \rightarrow V - W$, $V \notin \mathbf{Z}$ and $V' \rightarrow W$. Since $V'$ is reachable with the original algorithm and we continue onto an undirected edge, $V' \notin \mathbf{Z}$ and we arrive at $V'$ via an edge of the form $\leftarrow V'$ or $-V$. The original (and the modified) algorithm will therefore reach $W$, either via $\leftarrow V'\rightarrow W$ or $-V'\rightarrow W$. In either case, $W$ is reachable and if $W \notin \mathbf{Z}$ we will consider all vertices such that $W - W'$ or $W \rightarrow W'$ and if $W \in \mathbf{Z}$ we will consider all vertices such that $W \leftarrow W'$. This means $W'$ is reachable either way and we will in fact move onto a larger number of the adjacent nodes of $W$, regardless. The extra check we make in the modified algorithm therefore does not modify the output. 
\end{proof}

\subsection{Walk-Status in Reachability Algorithms}
\label{app:walkstatus}

In addition to the result from \cref{app:modifiedcriterion}
and
\cref{app:blockingincpdags}, we require one more idea in order to be able to verify the adjustment criterion with a reachability algorithm: when verifying the adjustment criterion a walk may at first not violate the adjustment criterion but as we append edges to it it may become a walk whose existence violates the adjustment criterion. For example, a directed walk starting from $\mathbf{T}$ does not violate the adjustment criterion until it either encounters a node in $\mathbf{Z}$ or turns into an open definite-status non-causal walk. In order to track such walks, we need to carry forward information about the current walk's status when traversing the graph;
specifically, we require that a quinary walk status be propagated forward.
Knowing a walk's status allows us to use more complex local rules about when to continue a walk
(for example, only stopping on a blocked walk when the walk is non-causal)
and assigning different tags to a node depending on the status of the walk with which we reached it;
two examples of this are
a)
adding a non-amenable tag to nodes we can reach with a possibly directed walk that begins with an undirected edge,
and
b)
adding a not-validly-adjusted-for tag to nodes we can reach with a possibly directed walk that contains a node in $\mathbf{Z}$, since this is a walk that contains a possibly causal node.

Assume we start the algorithm in $\mathbf{T}$ given some $\mathbf{Z}$,
by construction we never walk back into $\mathbf{T}$, that is, all walks are proper walks.
Also, we never visit the same node via the same edge on a walk of the same type twice,
such that our algorithm has runtime guarantee $O(p+m)$ analogous to the Bayes-Ball and gensearch algorithm.
The walk status is quinary and one of the following:
\begin{description}

\item[$\walkcausalopen$] – These are possibly directed walks that started with an edge pointing out of $T$ are not blocked by $\mathbf{Z}$.
Reaching a node $Y$ by a $\walkcausalopen$ walk does not tell us anything about whether $(\g,T,Y)$ is amenable or whether $\mathbf{Z}$ is a valid adjustment set for $(T,Y)$. 
Instead, we need to keep walking such a walk as it may turn into a blocked (possibly directed walk that started with an edge pointing out of $T$) walk upon passing through $\mathbf{Z}$ or a non-causal walk upon traversing along a backward-facing edge $\gets$, which are walks that contain information about amenability or validity of adjustment for the nodes reached.

\item[$\walknamopen$] – These are possibly directed walks that started with an undirected edge out of $T$ and are not blocked by $\mathbf{Z}$.
Reaching a node $Y$ by a $\walknamopen$ walk tells us that $(\g,T,Y)$ is not amenable (which implies that $\mathbf{Z}$ cannot be a valid adjustment set for $(T,Y)$) as Condition 1 in the Modified Adjustment Criterion is violated.
We need to keep walking such a walk as other nodes reached by it are also not amenable and as it may turn into a blocked (possibly directed walk that started with an undirected edge out of $T$) or a non-causal walk (which we need to check are blocked).

\item[$\walkcausalblocked$] – These are possibly directed walks that started with an edge pointing out of $T$ and contain a node in $\mathbf{Z}$.
Reaching a node $Y$ by a $\walkcausalblocked$ walk tells us that $\mathbf{Z}$ is not a valid adjustment set for $(T,Y)$ as Condition 2 in the Modified Adjustment Criterion is violated (the walk must have passed through a node in $\mathbf{Z}$).
We need to keep walking such a walk as $\mathbf{Z}$ is also not a valid adjustment set for other nodes reached by this walk and as it may turn into a non-causal walk (which we need to check are blocked).

\item[$\walknamblocked$] – These are possibly directed walks that started with an undirected edge out of $T$ and are blocked by $\mathbf{Z}$.
Reaching a node $Y$ by a $\walknamblocked$ walk tells us that $(\g,T,Y)$ is not amenable (which implies that $\mathbf{Z}$ cannot be a valid adjustment set for $(T,Y)$) as Condition 1 in the Modified Adjustment Criterion is violated.
We need to keep walking such a walk as other nodes reached by it are also not amenable and as it may turn into a non-causal walk (which we need to check are blocked).

\item[$\walknoncausal$] — These are walks that have passed through at least one backward-facing edge and are thus non-causal and are not blocked by $\mathbf{Z}$;
if they were blocked by $\mathbf{Z}$ we would just stop walking such a non-causal blocked walk.
Reaching a node $Y$ by a $\walknoncausal$ walk tells us that $\mathbf{Z}$ is not a valid adjustment set for $(T,Y)$ as Condition 3 in the Modified Adjustment criterion is violated.

\end{description}

To summarise and help intuition,
we provide the following illustration
of the possible walk-status changes:

\centerline{ %
\definecolor{colA}{HTML}{358132}
\definecolor{colB}{HTML}{6e2b2d}
\definecolor{colC}{HTML}{484BB7}
\begin{tikzpicture}
\node(cb) at(6.5,-1) {$\walkcausalblocked$};
\node[draw](co) at(0,-1) {$\walkcausalopen$};
\node(ci) at(0,-1.5) {{\color{colC}\small start $T \to W$}};
\node[draw](nc) at(-6.5,0) {$\walknoncausal$};
\node(nci) at(-8.5,0) {{\color{colC}\small start $T \gets W$}};
\node[draw](no) at(0,1) {$\walknamopen$};
\node(ni) at(0,1.5) {{\color{colC}\small start $T \arrow W$}};
\node(nb) at(6.5,1) {$\walknamblocked$};

\node[colA] at (-2.8,0) {when walking $\to V \gets W, V\in\mathbf{Z}$};
\node[colB] at (3.3,0) {when walking
$\begin{Bmatrix}\to V \to W \\ \arrow V \to W \\ \arrow V \arrow W \\ \to V \arrow W \end{Bmatrix}, V\in\mathbf{Z}$};

\draw[->,dashed,colA] (co) -| (nc);
\draw[->,dashed,colA] (no) -| (nc);

\draw[->,dashed,colB] (co) -- (cb);
\draw[->,dashed,colB] (no) -- (nb);
\end{tikzpicture}
}

An instructive first example of a reachability algorithm is \cref{algorithm: amenability} to check amenability, which we use also in our implementation of the identification strategies for CPDAGs.
Here, the routine simplifies considerably, since we only start walking $\walknamopen$ and $\walknamblocked$ walks from $T$ and all nodes $Y$ reached by a such a walk are nodes such that $(\g,T,Y)$ is not amenable.

Finally, in \cref{algorithm: VAS check} we present the key to efficiently implementing adjustment verification for our adjustment-based identification distances:
Given a graph (DAG or CPDAG) $\g$, treatment $T$, and candidate adjustment set $\mathbf{Z}$ with $T\notin \mathbf{Z}$,
\cref{algorithm: VAS check} returns in $O(p+m)$ time two lists  
a) \texttt{NAM} (``not amenable'') containing all $Y\notin \mathbf{T}$ such that $(\g,T,Y)$ is not amenable,
and b) \texttt{NVA} (``not validly adjusted for'') containing all $Y\notin\mathbf{T}\cup\mathbf{Z}$ such that $\mathbf{Z}$ is not a valid adjustment set for $(T,Y)$ in $\g$ and all $Y\in \mathbf{Z}$.

\vfill

\begin{algorithm}
\caption{Check amenability of a CPDAG $\G$ relative to $(\mathbf{T},Y)$ for a given set $\mathbf{T}$ of treatment nodes and all possible $Y$}
\label{algorithm: amenability}
\begin{algorithmic}[1]
\State \textbf{Input}: CPDAG $\G$ and a set of treatment nodes $\mathbf{T}$ in $\G$
\State \textbf{Output}: Set \texttt{NAM} of nodes $Y\not\in \mathbf{T}$ in $\G$ such that $\G$ is \textbf{n}ot \textbf{am}enable relative to $(T,Y)$
\vspace{.5em}
\Function{visit}{\texttt{arrivedby, $V$}}
        \State \texttt{visited.insert($V$)}
        \If{\texttt{arrivedby == init}}
        \Comment{Start walking proper possibly directed walks that do not start out of $\mathbf{T}$}
            \For{$W$ in $\adjacents(V)\setminus\mathbf{T}$}
                \If{\texttt{$W$ not in visited}}
                    \State \Call{visit}{\texttt{$\arrow$, $W$}}
                \EndIf
            \EndFor
        \Else
        \Comment{Continue walking proper possibly directed walks}
            \State \texttt{NAM.push($V$)}
            \Comment{Reached $V$ by a proper possibly directed walk that does not start out of $\mathbf{T}$}
                \For{$W$ in $\adjacents(V)\setminus\mathbf{T}$}
                    \If{\texttt{$W$ not in visited}}
                        \State \Call{visit}{\texttt{$\arrow$, $W$}}
                    \EndIf
                \EndFor
                \For{$W$ in $\ch(V)\setminus\mathbf{T}$}
                    \If{\texttt{$W$ not in visited}}
                        \State \Call{visit}{\texttt{$\to$, $W$}}
                    \EndIf
                \EndFor
        \EndIf
\EndFunction
\vspace{.5em}
\State Initialise \texttt{NAM} as empty set
\State Initialise \texttt{visited} as empty HashSet
\vspace{.5em}
\For{$V$ in $\mathbf{T}$}
  \Call{visit}{\texttt{init, $V$}}
\EndFor
\vspace{.5em}
\State \Return{\texttt{NAM}}
\end{algorithmic}
\end{algorithm}

\begin{algorithm}[ht]
\caption{Validate $\mathbf{Z}$ as adjustment set relative to $(\mathbf{T},Y)$ for a given set $\mathbf{T}$ of treatment nodes and all possible $Y$ in $\G$}
\label{algorithm: VAS check}
\begin{algorithmic}[1]
\State \textbf{Input}: CPDAG (or DAG) $\G$,
a set of treatment nodes $\mathbf{T}$ in $\G$,
and a set of adjustment nodes $\mathbf{Z}$ in $\G$ with $\mathbf{T}\cap\mathbf{Z}=\emptyset$
\State \textbf{Output}:
Set \texttt{NAM} of nodes $Y\not\in \mathbf{T}$ in $\G$ such that $\G$ is \textbf{n}ot \textbf{am}enable relative to $(T,Y)$
\State \phantom{\textbf{Output}:}
Set \texttt{NVA} of nodes $Y\not\in \mathbf{T}$ in $\G$ such that $\mathbf{Z}$ is \textbf{n}ot a \textbf{v}alid \textbf{a}djustment set for $(T,Y)$ in $\G$
\vspace{.5em}

\Function{nextsteps}{\texttt{arrivedby, $V$}}
\Comment{Return \texttt{(moveonby, $W$, blocked)} triplets}
\State Initialise \texttt{next} as empty set
\If{\texttt{arrivedby == $\to$}}
    \For{$W$ in $\pa(V)\setminus\mathbf{T}$}
    \Comment{collider $\to V \gets W$}
        \State \texttt{next.push(($\gets$, $W$, $\mathbb{1}(V \not\in \mathbf{Z})$))}
    \EndFor
\ElsIf{\texttt{arrivedby in $\{\texttt{init}, \gets\}$}}
    \For{$W$ in $\pa(V)\setminus\mathbf{T}$}
    \Comment{$\gets V \gets W$}
        \State \texttt{next.push(($\gets$, $W$, $\mathbb{1}(V \in \mathbf{Z})$))}
    \EndFor
\EndIf
    \For{$W$ in $\adjacents(V)\setminus\mathbf{T}$}
    \Comment{$\to V \arrow W$ or $\arrow V \arrow W$ or $\gets V \arrow W$}
        \State \texttt{next.push(($\arrow$, $W$, $\mathbb{1}(V \in \mathbf{Z})$))}
    \EndFor
\For{$W$ in $\ch(V)\setminus\mathbf{T}$}
\Comment{$\to V \to W$ or $\arrow V \to W$ or $\gets V \to W$}
    \State \texttt{next.push(($\to$, $W$, $\mathbb{1}(V \in \mathbf{Z})$))}
\EndFor
\State \Return \texttt{next}
\Comment{omits steps to $\mathbf{T}$ and $\arrow V \gets W$}
\EndFunction

\vspace{.5em}
\Function{visit}{\texttt{(arrivedby, $V$, walkstatus)}}
        \State \texttt{visited.insert((arrivedby, $V$, walkstatus))}
        \If{\texttt{walkstatus in $\{\walknamopen,\walknamblocked\}$}}
            \State \texttt{NAM.push($V$)} and \texttt{NVA.push($V$)}
            \Comment{Reached $V$ by a possibly directed walk that does not start out of $\mathbf{T}$}
        \ElsIf{\texttt{walkstatus == $\walknoncausal$}}
            \State \texttt{NVA.push($V$)}
            \Comment{Reached $V$ by a non-causal walk that is not blocked by $\mathbf{Z}$}
        \ElsIf{\texttt{walkstatus == $\walkcausalblocked$}}
            \State \texttt{NVA.push($V$)}
            \Comment{Reached $V$ by a possibly directed walk that is blocked by $\mathbf{Z}$}
        \EndIf
        \For{\texttt{(moveonby, $W$, blocked) in }\Call{nextsteps}{\texttt{arrivedby, $V$}}}
            \State \texttt{next = none}
            \If{\texttt{walkstatus == init}}
                \If{\texttt{moveonby == $\to$}}
                    {\texttt{next = ($\to$, $W$, \walkcausalopen)}}
                    \Comment{Start possibly directed walk $\textbf{T} \to$}
                \ElsIf{\texttt{moveonby == $\arrow$}}
                    {\texttt{next = ($\arrow$, $W$, \walknamopen)}}
                    \Comment{Start possibly directed walk $\textbf{T}\arrow$}
                \ElsIf{\texttt{moveonby == $\gets$}}
                    {\texttt{next = ($\gets$, $W$, \walknoncausal)}}
                    \Comment{Start non-causal walk}
                \EndIf
            \ElsIf{\texttt{walkstatus in $\{\walkcausalopen,\walkcausalblocked\}$}}
                \If{\texttt{moveonby in $\{\to,\arrow\}$}}
                    \If{\texttt{blocked == false}}
                        {\texttt{next = (moveonby, $W$, walkstatus)}}
                    \ElsIf{\texttt{blocked == true}}
                        {\texttt{next = (moveonby, $W$, $\walkcausalblocked$)}}
                    \EndIf
                \ElsIf{\texttt{moveonby == $\gets$} and \texttt{blocked == false} and \texttt{walkstatus == $\walkcausalopen$}}
                    \State {\texttt{next = (moveonby, $W$, $\walknoncausal$)}}
                \EndIf
            \ElsIf{\texttt{walkstatus in $\{\walknamopen,\walknamblocked\}$}}
                \If{\texttt{moveonby in $\{\to,\arrow\}$}}
                    \If{\texttt{blocked == false}}
                        {\texttt{next = (moveonby, $W$, walkstatus)}}
                    \ElsIf{\texttt{blocked == true}}
                        {\texttt{next = (moveonby, $W$, $\walknamblocked$)}}
                    \EndIf
                \ElsIf{\texttt{moveonby == $\gets$} and \texttt{blocked == false} and \texttt{walkstatus == $\walknamopen$}}
                    \State {\texttt{next = (moveonby, $W$, $\walknoncausal$)}}
                \EndIf
            \ElsIf{\texttt{walkstatus == $\walknoncausal$} and \texttt{blocked == false}}
                \State {\texttt{next = (moveonby, $W$, $\walknoncausal$)}}
            \EndIf
        \If{\texttt{next is not none} and \texttt{next not in visited}} 
            \Call{\texttt{visit}}{\texttt{next}}
        \EndIf
        \EndFor
\EndFunction
\vspace{.5em}
\State Initialise \texttt{NAM} as empty set
\State Initialise \texttt{NVA}=$\mathbf{Z}$
\State Initialise \texttt{visited} as empty HashSet
\vspace{.5em}
\For{$V$ in $\mathbf{T}$}
  \Call{visit}{\texttt{(init, $V$, init)}}
\EndFor
\vspace{.5em}
\State \Return{\texttt{NAM} and \texttt{NVA}}
\end{algorithmic}
\end{algorithm}

\FloatBarrier

\subsection{Our Algorithm Enables Efficient Calculation of Parent-, Ancestor-, and Oset-AID}
\label{app:efficientalg}

For the three distances, the Parent-AID, Ancestor-AID, and Oset-AID,
we need to identify adjustment sets in $\Gguess$ with an additional amenability check in case $\Gguess$ is a CPDAG
and then verify the proposed adjustment sets in $\Gtrue$.
While 
algorithms
with
$O(p+m)$ runtime
exist
for each involved computation,
the algorithmic development in the preceding subsections is crucial to enable efficient calculation of the distances:
\cref{algorithm: VAS check} enables us to verify adjustment sets for
all
$\{(T,Y') \mid Y'\in \mathbf{V}\setminus\{T\}\}$
with one $O(p+m)$ run, instead of performing $(p-1)$ separate runs of a valid adjustment verifier algorithm.
For simplicity and as an instructive example,
we first discuss our implementation of the Parent-AID for DAGs
and then present the general routine for implementing our distances.

\subsubsection{Calculating the Parent Adjustment Identification Distance Efficiently}

To calculate the Parent-AID between two DAGs $\Gtrue$ and $\Gguess$ over $p$ nodes and $m$ edges,
we need to iterate over all tuples $(T,Y)$ of nodes,
obtain the parent set of the treatment in $\Gguess$,
and check whether this set is a valid adjustment set in $\Gtrue$ with respect to $(T,Y)$.
For their SID implementation, \cite{peters2015structural} report a worst-case runtime of
$O(p \cdot \log_2(p) \cdot p^3)$
where the factor $p^3$ corresponds to squaring of the adjacency matrix of $\Gtrue$
which is done $\lceil\log_2(p)\rceil$ times to assemble a path matrix that codes which nodes
are
reachable 
from each of the $n$ treatment nodes.\footnote{%
One may be able to reduce the cubic runtime for the matrix multiplication
if the adjacency matrices exhibit extra known structure,
though, the algorithm with the best known asymptotic runtime to date of $O(p^{2.37})$
is a galactic algorithm and not usable in practice \citep{alman2020refined}. %
For certain adjacency matrices, the Strassen algorithm for matrix multiplication may enable a reduction to $O(p^{\log_2(7)}) \approx O(p^{2.8})$.
}

Combining the above algorithms, we can calculate the Parent-AID with an algorithm
with runtime $O(p(p+m))$
as follows:

\begin{itemize}
\item Initialise the mistake count $c = 0$
\item For each node $T$ (each of the following steps can be completed in $O(p + m)$ time)
\begin{itemize}
\item Obtain $\mathbf{Z}$ as the set of parents of $T$ in $\Gguess$
\item Obtain $\mathbf{ND}$ as the set of non-descendants of $T$ in $\Gtrue$
\item Obtain $\mathbf{NVA}$ as the the set of nodes $Y$ such that $\mathbf{Z}$ is not a valid adjustment set for $(T,Y)$ in $\Gtrue$
\item Add 
\[
\underbrace{|\mathbf{Z} \setminus \mathbf{ND}|}_{\text{guessed no effect, but descendant in $\Gtrue$}}
+
\underbrace{|\mathbf{Z}^\complement \cap \mathbf{NVA}|}_{\mathbf{Z}\text{ valid adjustment set in $\Gguess$, but not in $\Gtrue$}}
\]
to the mistake count $c$
\end{itemize}
\item Return $d^{\mathcal{I}_P}(\Gtrue, \Gguess) = c$
\end{itemize}

Our Parent-AID coincides with the SID only as distance between DAGs, but, in contrast to the SID, generalizes to CPDAGs.
The multi-set SID between CPDAGs requires exponential runtime,
while the Parent-AID between CPDAGs is still $O(p(p+m))$ as shown in the next subsection.

\subsubsection{Calculating Adjustment Identification Distances Efficiently}

We fix the treatment $T$ and apply our algorithm to all tuples $(T,Y), T \neq Y$ simultaneously. We also group our identifying formulas as follow: For each $T$, the identification strategy algorithm 
returns a vector of $($node, identifying formula$)$ tuples
which we code as a triple $(\mathbf{A},\mathbf{B},(Y,\mathbf{Z}(Y))_{Y \in \mathbf{C}})$ consisting of a) the set nodes $\mathbf{A}$ for which the causal effect from $T$ is not identifiable, b) a set of nodes $\mathbf{B}$ for which the causal effect from $T$ is zero and c) a set of two-tuples consisting of the remaining nodes $Y \in \mathbf{C}=(\mathbf{A} \cup \mathbf{B})^c$ and for each $Y$ a corresponding valid adjustment set $\mathbf{Z}(Y)$. To compute this vector, we first apply a reachability algorithm to compute $\mathbf{I} = \mathrm{amen}(T,\Gguess)$ (\cref{algorithm: amenability}), where $\mathrm{amen}(T,\g)$ denotes the set of nodes $Y$ such that $(\g,T,Y)$ is amenable . For the parent strategy we then compute $\mathbf{P}=\pa(T,\Gguess)$ and return $( \mathbf{I}^c,\mathbf{P},(Y,\mathbf{P})_{Y \in \mathbf{I}\setminus \mathbf{P}})$. For the ancestor strategy we compute $\mathbf{A}=\an(T,\g)$, $\mathbf{D}=\de(T,\g)$ and return $(\mathbf{I}^c,\mathbf{D}^c \cap \mathbf{I},(N,\mathbf{A})_{N \in \mathbf{D}})$. For the Oset strategy, we compute $\mathbf{D}$, $\opt{\g}=\pa(\mathbf{D} \cap \an(Y,\G)) \setminus \mathbf{D}$ for each $Y \in \mathbf{D}$ and return $(\mathbf{I}^c, \mathbf{D}^c \cap \mathbf{I},(Y,\mathbf{O}(T,Y,\Gguess))_{Y \in \mathbf{D}})$. 
We repeat these steps for each $T$ and return a vector of three-tuples. The overall complexity is therefore $O(p(p+m))$ for the parent and ancestor strategies, and $O(p^2(p+m))$ for the Oset strategy. The additional $p$ is due to $\opt{\g}$ depending on $Y$, whereas $\pa(T,\g)$ and $\an(T,\g)$ do not depend on $Y$.

Consider now the verification step for a fixed treatment $T$ and the corresponding triple $(\mathbf{A},\mathbf{B},(Y,\mathbf{Z}(Y))_{Y \in \mathbf{C}})$.
We compute $\mathbf{I}' = \mathrm{amen}(T,\Gtrue)$ and $\mathbf{D}'=\de(Y,\Gtrue)$ in $O(p+m)$ using reachability algorithms. For each unique $\mathbf{Z}(Y)$ we then apply a reachability algorithm to compute set of nodes $N \in \mathbf{V}(T,\mathbf{Z}(Y),\Gtrue)$ such that $\mathbf{Z}(Y)$ is a valid adjustment set relative to $(T,N)$ in $\Gtrue$ (\cref{algorithm: VAS check}).
For the parent and ancestor strategy there is only one $\mathbf{Z}(Y)$, so we only need to do this step once.
We then add 
$$|\mathbf{A} \cap \mathbf{I}'| + |\mathbf{B} \cap \mathbf{D}'| + \#\{Y \in \mathbf{C} \setminus  \mathbf{V}(T,\mathbf{Z}(Y),\Gtrue)\}$$
to the distance between $\Gtrue$ and $\Gguess$. As we have to repeat this for each $T$, the overall runtime for the verifier is $O(p(p+m))$ for the parent and ancestor strategies, and $O(p^2(p+m))$ for the Oset strategy.

\section{Empirical Analysis of Algorithmic Time Complexity}\label{app:complexityexperiment}

To empirically analyze the runtime
complexity of our distance implementations,
we evaluate the algorithms on inputs of varying size $p$ and measure the 
runtime
$r_{\mathrm{emp}}(p)$
(here, we use the wall time).
For a given complexity, such as $O(p^2)$,
we then project the runtime we would expect for any $p$ based on the runtime observed for the smallest size $p$. We denote the projected runtime for $p$ as $r_{\mathrm{proj}}(p)$.
The idea is that the ratio of the projected runtime $r_{\mathrm{proj}}(p)$ and the observed empirical runtime $r_{\mathrm{emp}}(p)$
approaches $1$ in the limit of $p\to\infty$ if the implementation has the complexity used to compute $r_{\mathrm{proj}}(p)$.
For a given algorithm $\texttt{distance}(\Gtrue, \Gguess)$ we proceed as follows.

\begin{itemize}

\item \emph{(Grid of graph sizes)}\\
Specify a grid of graph sizes $P$, for example, $P=(8, 16, 32, 64, 128, 256, 512, 1024)$.

\item \emph{(Observed empirical runtimes)}\\
Sample, for each $p\in P$, $11$ pairs of DAGs
$\Gtrue=(\mathbf{V},\mathbf{E}_\text{true})$
and
$\Gguess=(\mathbf{V},\mathbf{E}_\text{guess})$
with $|\mathbf{V}|=p$ and edges drawn independently with
probability $20/(p-1)$ for sparse and $0.3$ for dense graphs
(the expected number of edges is thus linear in the number of nodes for sparse, and quadratic for dense graphs). We run $\texttt{distance}$ on those $11$ pairs and compute the average runtime over these $11$ runs, denoted $r_{\mathrm{emp}}(p)$.

\item \emph{(Project the runtime under the given time complexity $O(c(p))$ where, for example $c(p)=p^2$)}\\
We obtain the projected runtime for inputs of size $p$ under the given time complexity based on the smallest input size $p^*=\min(P)$ as
\[
\text{projected runtime:}\quad r_{\mathrm{proj}}(p) = c(p)\frac{r_{\mathrm{emp}}(p^*)}{c(p^*)}.
\]

\item \emph{(Relative projected runtime)}\\
To compare the projected runtime to the observed empirical runtime,
we then visualize the projected runtime as a fraction of the observed empirical runtime
\[
\text{relative projected runtime:}\quad \frac{r_{\mathrm{proj}}(p)}{r_{\mathrm{emp}}(p)}.
\]

\end{itemize}

As we consider asymptotic complexity,
we need to evaluate for large enough $p$
to assess the following trends of relative projected runtimes for increasing $p$.
If, when comparing to $O(c(p))$, the
relative projected runtime increases with the number of nodes $p$,
this indicates that the empirical time complexity is lower than $O(c(p))$.
If, when comparing to $O(c(p))$, the
relative projected runtime decreases with the number of nodes $p$,
this indicates that the empirical time complexity is larger than $O(c(p))$.
If the algorithm has time complexity $O(c(p))$
then we expect the relative projected runtime for this complexity
to be close to constant.

\section{Distance Comparison}\label{app:comparison}

The simulation study described in \cref{sec: distance comparison} is part of a larger study in which we consider $8$ parameter settings.
Specifically, we consider all possible combinations of the following three choices:
i) $\Gtrue$ is a dense, respectively sparse graph randomly drawn as described in \cref{sec: empirical runtime},
ii) $\Gguess$ is $\Gtrue$ with one edge randomly removed, respectively $\Gguess$ is randomly drawn in the same way as $\Gtrue$ and
iii) $\Gtrue$ and $\Gguess$ are $30$-node, respectively $100$-node graphs. For each of these $8$ settings we randomly draw $300$ pairs of $\Gtrue$ and $\Gguess$ graphs. For each pair we compute the three proposed adjustment identification distances and the SHD.
To summarise the results we obtain correlation tables across the distances for each of the $8$ experiments.
We also provide corresponding scatter plots; to reduce clutter we do so only for the $30$-node graphs.

\subsection{Distances between a random graph and a graph with one edge removed} \label{app:removal}

\cref{tab:edge-removal-dense} contains correlation tables for the setting that $\Gtrue$ is a random dense graph and $\Gguess$ is $\Gtrue$ with one edge removed as well as the average value for each distance over the $300$ pairs. The left table contains the correlations for the $p=30$ case and the right those for the $p=100$ case. We also provide a corresponding scatter plot for the $p=30$ case in \cref{fig:dense-removal}. We do not include the SHD, as the SHD between $\Gtrue$ and $\Gguess$ is by construction $1$.

When $\Gguess$ graphs are close to the true graphs $\Gtrue$ the correlation between the three distances is surprisingly small; particularly the correlations between the Oset-AID, respectively the Ancestor-AID and the Parent-AID are small. This may be explained by how the Parent-AID treats node tuples $(T,Y)$ for which $Y \notin \de(T,\Gguess)$ differently from the Ancestor- and Oset-AID: For such tuples, both the Oset and ancestor adjustment strategy return $f(y)$, whereas parent adjustment only does so if $Y \in \pa(T,\Gguess)$. As a result, small differences between $\Gtrue$ and $\Gguess$ tend to lead to a larger number of mistakes if we apply parent adjustment as opposed to Oset or ancestor adjustment. This is also reflected in the larger average of the Parent-AID. Another interesting pattern visible in the scatter plot, is the number of cases where a large Parent-AID coincides with a small or even zero Ancestor-AID. This illustrates how $\Gguess$ graphs that respect the causal orders of $\Gtrue$ may nonetheless be deemed very distant by the Parent-AID. Overall, the results indicate that the three distances are meaningfully different, that is, they capture distinct information.

\begin{table}[h]
\centering
\small{
\begin{tabular}{l|lll}
\toprule
$p=30$ & Ancestor-AID & Oset-AID & Parent-AID \\
\midrule
Ancestor-AID & 1 & 0.7281 & 0.0886 \\
Oset-AID & 0.7281 & 1 & 0.2080 \\
Parent-AID & 0.0886 & 0.2080 & 1 \\
\bottomrule
Average dist. & 2.0 & 5.9 & 11.2 \\
\end{tabular}
\hfill
\begin{tabular}{l|lll}
\toprule
$p=100$ & Ancestor-AID & Oset-AID & Parent-AID \\
\midrule
Ancestor-AID & 1 & 0.7441 & 0.3717 \\
Oset-AID & 0.7441 & 1 & 0.3114 \\
Parent-AID & 0.3717 & 0.3114 & 1 \\
\bottomrule
Average dist. & 3.4 & 13.6 & 39.8 \\
\end{tabular}}
    \caption{Correlation tables for the case that $\Gtrue$ is a random dense graph and $\Gguess$ is $\Gtrue$ with one edge removed.}
    \label{tab:edge-removal-dense}
\end{table}

\cref{tab:edge-removal-sparse} contains correlation tables for the setting that $\Gtrue$ is a random sparse graph and $\Gguess$ is $\Gtrue$ with one edge removed as well as the average value for each distance over the $300$ pairs. The left table contains the correlations for the $p=30$ case and the right those for the $p=100$ case. We also provide a corresponding scatter plot for the $p=30$ case in \cref{fig:sparse-removal}. We do not include the SHD, as the SHD between $\Gtrue$ and $\Gguess$ is by construction $1$.

The results are overall qualitatively similar to the results for dense graphs, which indicates that the distinct characteristics of the three distances are not specific to sparse or small graphs but are in fact a characteristic of the three distances.

\begin{table}[h!]
\centering
\small{
\begin{tabular}{l|lll}
\toprule
$p=30$ & Ancestor-AID & Oset-AID & Parent-AID \\
\midrule
Ancestor-AID & 1 & 0.8564 & 0.4377 \\
Oset-AID & 0.8564 & 1 & 0.3749 \\
Parent-AID & 0.4377 & 0.3749 & 1 \\
\bottomrule
Average dist. & 1.0& 2.4& 8.3\\
\end{tabular}
\hfill
\begin{tabular}{l|lll}
\toprule
$p=100$ & Ancestor-AID & Oset-AID & Parent-AID \\
\midrule
Ancestor-AID & 1 & 0.7280 & 0.3019 \\
Oset-AID & 0.7280 & 1 & 0.2685 \\
Parent-AID & 0.3019 & 0.2685 & 1 \\
\bottomrule
Average dist. & 4.1& 24.8& 42.7 \\
\end{tabular}}
    \caption{Correlation tables for the case that $\Gtrue$ is a random sparse graph and $\Gguess$ is $\Gtrue$ with one edge removed.}
    \label{tab:edge-removal-sparse}
\end{table}

\clearpage

\begin{figure}
\centering
\resizebox{!}{.44\textheight}{\includegraphics[draft=false]{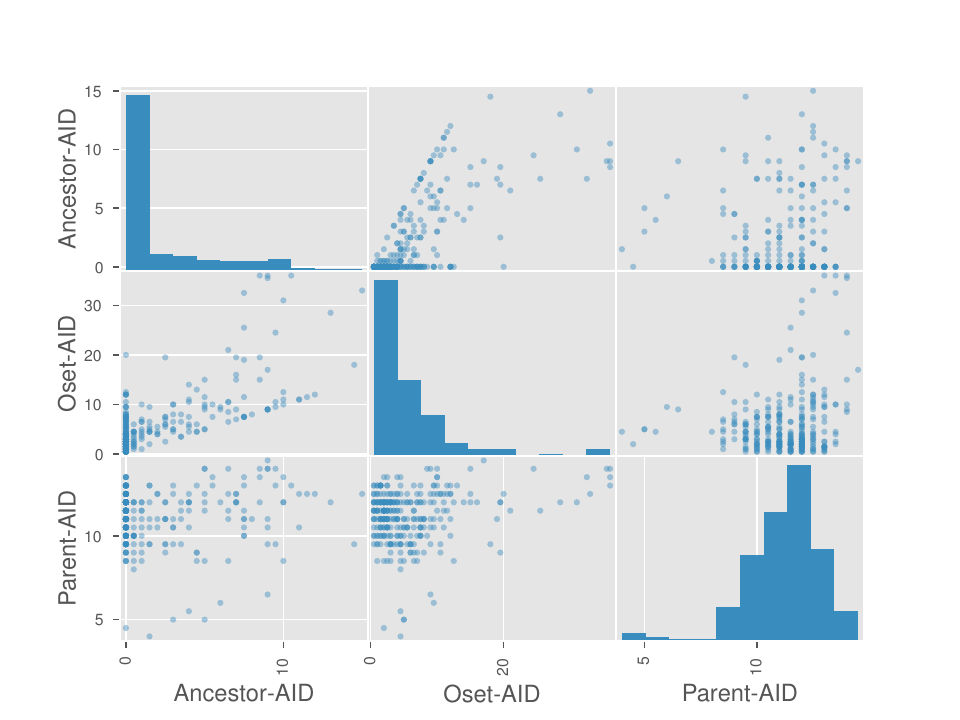}}
    \caption{Scatter plot for the case that $\Gtrue$ is a random $30$-node dense graph and $\Gguess$ is $\Gtrue$ with one edge removed.}
    \label{fig:dense-removal}
\end{figure}

\begin{figure}
    \centering
\resizebox{!}{.44\textheight}{\includegraphics[draft=false]{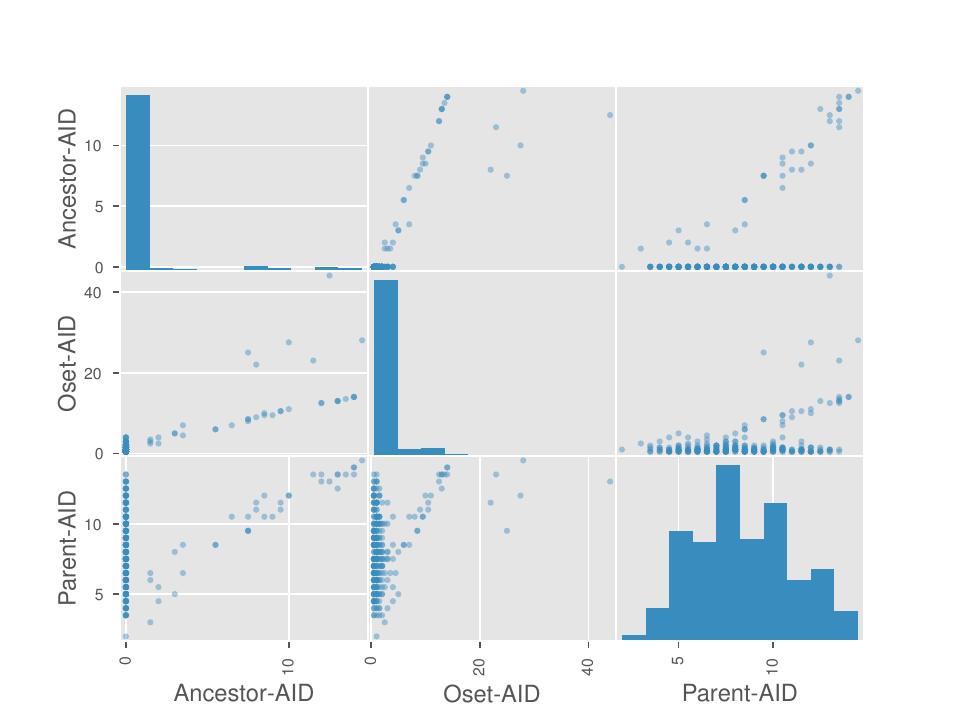}}
    \caption{Scatter plot for the case that $\Gtrue$ is a random $30$-node sparse graph and $\Gguess$ is $\Gtrue$ with one edge removed.}
    \label{fig:sparse-removal}
\end{figure}

\clearpage

\subsection{Distances between two random graphs} \label{app:random}

\cref{tab:random-dense} contains correlation tables for the setting that $\Gtrue$ and $\Gguess$ are random dense graphs. The left table contains the correlations for the $p=30$ case and the right those for the $p=100$ case. We also provide a corresponding scatter plot for the $p=30$ case in \cref{fig:dense-random}.

When comparing two independently drawn random dense graphs, the distances are more strongly correlated than what was observed in \cref{app:removal}. Especially, the Ancestor-AID and the Oset-AID are very strongly correlated.
Overall, the results indicate that, while the distances have distinct characteristics, 
the practical difference may be less relevant when comparing vastly different graphs
as opposed to close-by graphs as in \cref{app:removal}.

\begin{table}[h]
\centering
\scriptsize{
\begin{tabular}{l|lllr}
\toprule
$p=30$ & Ancestor-AID & Oset-AID & Parent-AID & SHD \\
\midrule
Ancestor-AID & 1 & 0.9474 & 0.7429 & 0.4673 \\
Oset-AID & 0.9474 & 1 & 0.5715 & 0.5233 \\
Parent-AID & 0.7429 & 0.5715 & 1 & 0.1769 \\
SHD & 0.4673 & 0.5233 & 0.1769 & 1 \\
\bottomrule
Average dist. &253.6 & 258.7 & 383.0 & 202.4\\
\end{tabular}
\hfill
\begin{tabular}{l|lllr}
\toprule
$p=100$ & Ancestor-AID & Oset-AID & Parent-AID & SHD \\
\midrule
Ancestor-AID & 1 & 0.9819 & 0.8512 & 0.4038 \\
Oset-AID & 0.9819 & 1 & 0.7704 & 0.3838 \\
Parent-AID & 0.8512 & 0.7704 & 1 & 0.2959 \\
SHD & 0.4038 & 0.3838 & 0.2959 & 1 \\
\bottomrule
Average dist. & 3469.8& 3475.6& 4539.0&2299.7\\
\end{tabular}}
    \caption{Correlation tables for the case that $\Gtrue$ and $\Gguess$ are random dense graphs.}
    \label{tab:random-dense}
\end{table}

\cref{tab:random-sparse} contains correlation tables for the setting that $\Gtrue$ and $\Gguess$ are random sparse graphs. The left table contains the correlations for the $p=30$ case and the right those for the $p=100$ case. We also provide a corresponding scatter plot for the $p=30$ case in \cref{fig:sparse-random}.

The results for sparse graphs are overall qualitatively similar to the results for dense graphs.
One notable difference is the scatter plot between the Ancestor-AID and the Oset-AID which shows that in many cases the Ancestor-AID and the Oset-AID are the same but that in the cases where they differ they are rarely just slightly different. We are uncertain what drives this behavior and why it is less pronounced in dense graphs. 
One potential explanation is that the Ancestor-AID and the Oset-AID count a mistake whenever two variables are in ancestral relationship in $\Gguess$ but not in $\Gtrue$ or vice versa. Similarly, they do not count a mistake when two variables that are not in ancestral relationship in $\Gguess$ also are not in ancestral relationship in $\Gtrue$.
When comparing two random sparse graphs, this behavior may cover the large majority of node pairs and therefore the Ancestor-AID and the Oset-AID between random sparse graphs are often very similar.
The distances only possibly differ for node pairs $(T,Y)$ where $T$ is ancestor of $Y$ in both graphs
and only one of the two adjustment strategies applied to $\Gguess$ yields an adjustment set that is also a valid adjustment set in $\Gtrue$.

\begin{table}[h]
\centering
\scriptsize{
\begin{tabular}{l|lllr}
\toprule
$p=30$ & Ancestor-AID & Oset-AID & Parent-AID & SHD \\
\midrule
Ancestor-AID & 1 & 0.9782 & 0.9390 & 0.8170 \\
Oset-AID & 0.9782 & 1 & 0.9000 & 0.8264 \\
Parent-AID & 0.9390 & 0.9000 & 1 & 0.8026 \\
SHD & 0.8170 & 0.8264 & 0.8026 & 1 \\
\bottomrule
Average dist. & 316.0& 316.8& 356.0& 289.2\\
\end{tabular}
\hfill
\begin{tabular}{l|lllr}
\toprule
$p=100$ & Ancestor-AID & Oset-AID & Parent-AID & SHD \\
\midrule
Ancestor-AID & 1 & 0.9668 & 0.6965 & 0.3124 \\
Oset-AID & 0.9668 & 1 & 0.5212 & 0.2914 \\
Parent-AID & 0.6965 & 0.5212 & 1 & 0.1324 \\
SHD & 0.3124 & 0.2914 & 0.1324 & 1 \\
\bottomrule
Average dist. & 3229.3& 3242.9& 4639.4& 1694.5\\
\end{tabular}}
    \caption{Correlation tables for the case that $\Gtrue$ and $\Gguess$ are random sparse graphs.}
    \label{tab:random-sparse}
\end{table}

\clearpage

\begin{figure}
\centering
\resizebox{!}{.44\textheight}{\includegraphics[draft=false]{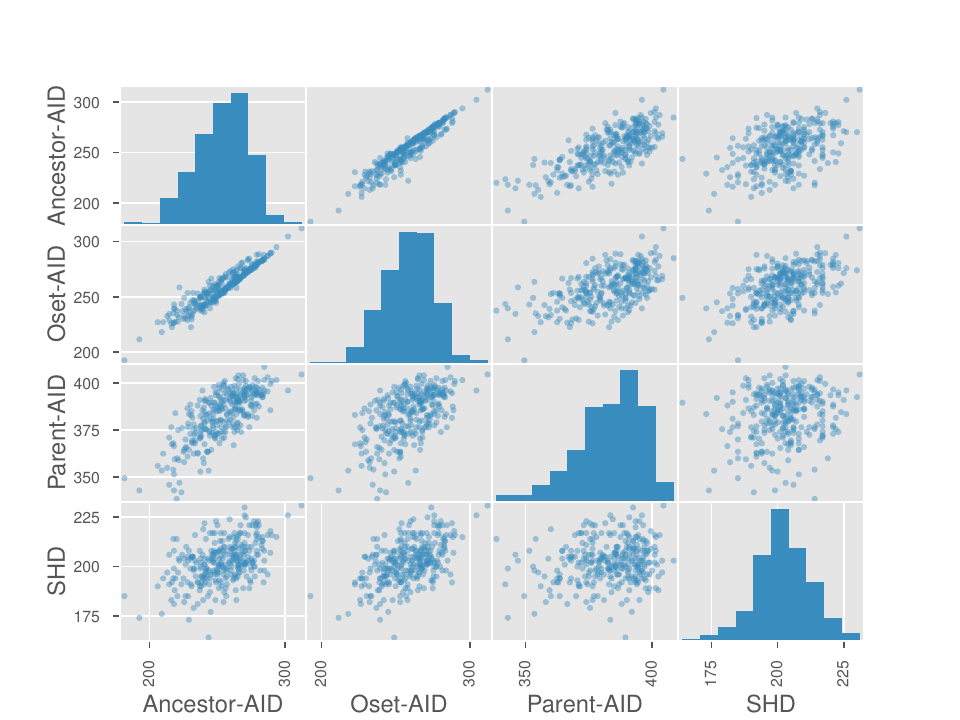}}
    \caption{Scatter plot for the case that $\Gtrue$ and $\Gguess$ are random dense graphs with $30$ nodes.}
    \label{fig:dense-random}
\end{figure}

\begin{figure}
\centering
\resizebox{!}{.44\textheight}{\includegraphics[draft=false]{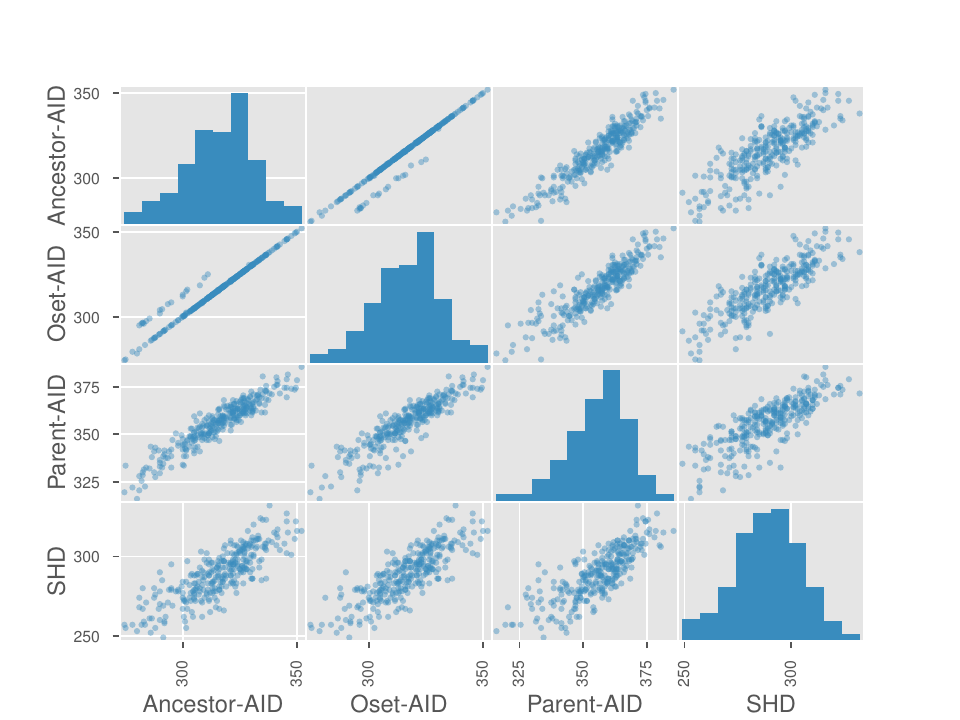}}
    \caption{Scatter plot for the case that $\Gtrue$ and $\Gguess$ are random sparse graphs with $30$ nodes.}
    \label{fig:sparse-random}
\end{figure}

\end{document}